\documentclass[accepted]{uai2026} 
                        
\usepackage[american]{babel}

\usepackage{natbib} 
\usepackage{mathtools} 
\usepackage{booktabs} 
\usepackage{tikz} 

\usepackage{amsmath}
\usepackage{amssymb}
\usepackage{mathtools}
\usepackage{amsthm}
\usepackage{algorithm}
\usepackage{algorithmic}

\usepackage{multirow}
\usepackage{booktabs}
\theoremstyle{plain}
\newtheorem{theorem}{Theorem}[section]

\newtheorem{lemma}[theorem]{Lemma}

\theoremstyle{definition}

\newtheorem{assumption}[theorem]{Assumption}
\theoremstyle{remark}

\usepackage{amsmath, amssymb, amsthm} 

\title{Model Merging as Probabilistic Inference in Fine-Tuning Parameter Space}

\author[1]{\href{mailto:<trongminhlong.bui@wsu.edu>}{Long Minh Bui}{}}
\author[2]{Tuan Anh Le Van}
\author[2]{Tung Phi Duc}
\author[2]{Phi Le Nguyen}
\author[1]{Jana Doppa}
\author[1]{Trong Nghia Hoang}
\affil[1]{%
    Washington State University
}
\affil[2]{%
    Hanoi University of Science and Technology
}

\begin{document}
\maketitle

\begin{abstract}
\vspace{-6mm}
Model merging aims to combine existing single-task solutions into a multi-task solution without additional data-driven fine-tuning.~Most existing approaches achieve this using geometric properties of local solution spaces.~However, such geometric views provide limited guidance for scoring how statistically useful each task-specific update direction is across tasks during merging.~To address this, we formulate model merging from a new perspective of probabilistic inference under a product-of-experts (PoE) scenario where each single-task solution defines an energy-based expert model (EBM) over the merged parameters.~We show that several existing model merging methods arise as special cases of our framework under energy designs that impose implicit Gaussian assumptions on directional residuals between merged and task-specific models.~Empirically, we find that these residuals are often heavy-tailed which exposes a mismatch with the imposed light-tailed Gaussian structures.~We address this with a heavy-tailed PoE design based on Cauchy experts, which better captures the observed residual behavior while admitting a provably convergent inference procedure.~Experiments across multiple tasks and architectures show significant improvements over state-of-the-arts baselines.~Our code is available at \url{https://github.com/MinhLong210/PoE-EBM-Merging.git}.\vspace{-3mm}
\end{abstract}

\section{Introduction}
\label{sec:intro}
Large pre-trained foundation models and their task-specific fine-tuned variants~\citep{achiam2023gpt, touvron2023llama} have become increasingly available for a wide range of downstream tasks.~The growing availability of such specialized models has motivated model merging, which seeks to combine multiple task-specific models into a single multi-task model without additional data-driven fine-tuning~\citep{hoang2019collective,NghiaNeurIPS19,NghiaICML20,NghiaICML21,yang2024model, li2023deep,hoang2024fewshot}.~For example, monolingual models can be combined to obtain a single multilingual model~\citep{ahmadian2024mix}.~Such model merging approaches are particularly valuable in many real-world production systems~\citep{Su2018ARS} where both local datasets and training pipelines cannot be centralized and synchronized~\footnote{Federated Learning~\citep{McMahan2016CommunicationEfficientLO} can help address data privacy but still requires synchronized local training processes.}.~Despite requiring no additional training data, model merging can often achieve performance competitive with full multi-task fine-tuning, which might be impractical in such settings.~Model merging is also attractive when storage resources are limited, such as on edge devices~\citep{voghoei2018deep, narayanswamy2024scaling}, or access to privacy-sensitive task-specific data is restricted~\citep{liang2025vision, zhang2024challenges, pan2024federated}.


{\bf Prior Work.} A common paradigm in model merging is to represent each task-specific model with a fine-tuning module, such as a low-rank adaptation (LoRA) matrix~\citep{hu2022lora} or a task vector capturing the difference between fine-tuned and pre-trained parameters~\citep{ilharco2022editing}. Model merging then reduces to aggregating these task-specific updates into a single multi-task model. Most existing approaches perform this aggregation using geometric properties of local solution spaces.

The simplest methods, including weight averaging and task arithmetic~\citep{wortsman2022robust,ilharco2022editing}, assume that task-specific updates lie in a common solution manifold and can be directly combined.~More sophisticated approaches, such as Fisher-weighted averaging~\citep{matena2022merging} and Gram-based weighting~\citep{jin2022dataless}, incorporate curvature or data-dependent geometric information to better account for differences among local solution spaces during aggregation. Another line of work explicitly seeks to align local solution spaces before merging. For example, DOGE~\citep{wei2025modeling} constructs a shared tangent space, while KnOTS~\citep{stoica2024model} derives a common low-dimensional subspace using singular value decomposition. Task-specific models are then projected onto these aligned representations prior to aggregation.

\begin{figure*}[t]
    \centering
    \includegraphics[width=\linewidth]{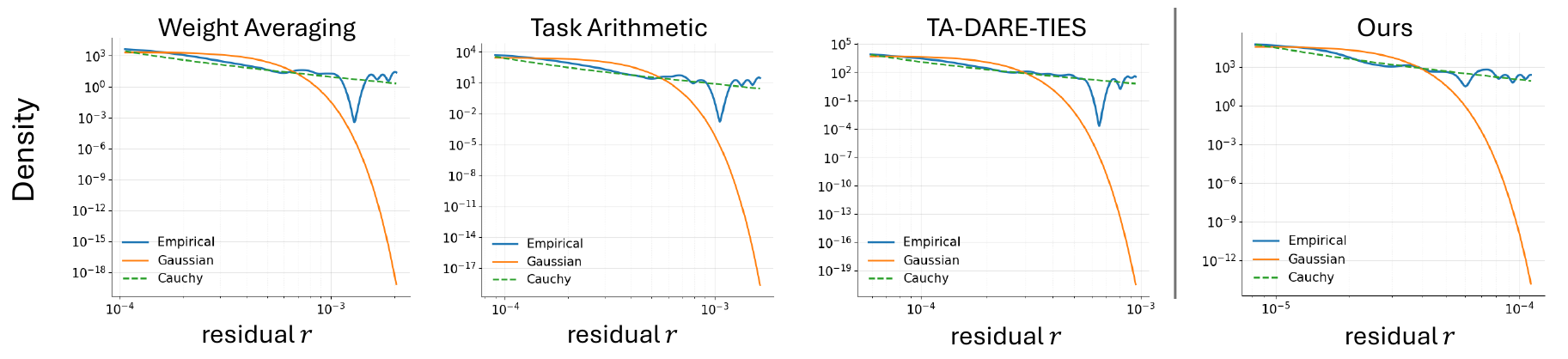}
    \caption{Empirical distributions of the directional residual \(r=(\boldsymbol{\zeta}-\boldsymbol{\theta})^\top\boldsymbol{\theta}\), which measures the drift of a merged update \(\boldsymbol{\zeta}\) from a task update \(\boldsymbol{\theta}\) along that task's own update direction.~We compare residuals produced by different merging methods when merging 7 fine-tuned ViT-L/14 models.~Log-scale density plots show that empirical tails decay substantially more slowly than the fitted Gaussian and align more closely with a fitted Cauchy distribution, indicating pronounced heavy-tailed residual behavior in model merging, revealing a mismatch with the (implicit) Gaussian expert structure in prior work.
}\vspace{-4mm}
    \label{fig:heavy_tail}
\end{figure*}

{\bf Limitation.}~These approaches largely treat each task-specific fine-tuning module as a deterministic point estimate and aggregate them with some geometry-guided operations.~However, this deterministic view, provides little information regarding how useful each update direction is across different tasks.~As a result, update directions that are effective only for individual tasks can be aggregated with directions that are consistently useful across tasks.~Such task-specific directions can then cancel or dominate useful cross-task directions, pulling the merged update away from shared directions that benefit multiple tasks.~This suggests that model merging should not only aggregate update directions, but also estimate how confidently each direction should influence the merged module.



{\bf Motivation and Solution Vision.}~Motivated by the above intuition, we investigate model merging within a broader probabilistic framework in which each task update induces a distribution over candidate shared update directions rather than a single point estimate.~The local probabilistic score assigned for each candidate direction thus reflects how it is supported by the corresponding task.~In this view, merging becomes evidence aggregation.~Directions supported by multiple tasks receive higher aggregate confidence while directions supported only by individual tasks receive low support from the remaining tasks and are down-weighted.

We instantiate this idea by casting model merging as MAP inference in the fine-tuning parameter space under a product of task-specific energy-based experts.~This view recovers existing merging rules as special cases under particular energy designs and exposes the uncertainty assumptions they implicitly impose.~In particular, quadratic energies recover classical averaging-style methods and correspond to Gaussian experts.~This approach is also closely related to a recent uncertainty-aware gradient matching method of~\citet{daheim2023model} which imposes an implicit Gaussian expert design via its Laplace approximation interpretation.~While this method provides important empirical evidence for the effectiveness of probabilistic model merging, our analysis reveals that its implicit Gaussian structure does not fit well with the distribution over directional residuals between individual and merged updates.~In particular, we show that these directional residuals exhibit substantial heavy-tailed behavior while Gaussian models are inherently light-tailed (see Fig.~\ref{fig:heavy_tail}).~To address this limitation, we develop a product-of-experts (PoE) formulation with heavy-tailed energy-based expert models (EBM).~Our solution approach is substantiated by the following technical contributions:

\textbf{1.~A unified probabilistic model merging framework.}~We formulate model merging as MAP inference in the fine-tuning parameter space under a product of task-specific energy-based experts.~In this view, each fine-tuning module induces an energy over the directional residual between a candidate merged update and the task-specific update.~The merged module then corresponds to the MAP estimate under the resulting product of experts.~We show that existing merging rules, including uniform averaging and Fisher-weighted averaging, arise as special cases under particular energy designs.~This provides a unified lens for these methods and exposes the implicit probabilistic assumptions imposed by their aggregation rules (Section~\ref{sec: method} and~\ref{sec:design}).

\textbf{2.~Heavy-tailed merging with convergence guarantees.}
Under the above probabilistic framework, we show that existing merging methods correspond to light-tailed distributions over the directional residual
\(r = (\boldsymbol{\zeta}-\boldsymbol{\theta})^\top \boldsymbol{\theta}\),
which measures how far a candidate merged update \(\boldsymbol{\zeta}\) drifts from task update \(\boldsymbol{\theta}\) along that task's own update direction.~Our empirical analysis shows that these residuals exhibit heavy-tailed behavior, revealing a mismatch with the Gaussian expert structure often implicitly assumed in prior work.~To address this, we develop a novel heavy-tailed Cauchy-based expert designs and an efficient fixed-point MAP inference algorithm with convergence guarantees (Section~\ref{sec: merging algo}).

\textbf{3.~Evaluation across vision and language models.}~We evaluate the proposed framework across diverse vision and language benchmarks spanning multiple model families.~Our approach consistently improves merged model performance over state-of-the-art (SOTA) baselines (Section~\ref{sec: Experiments}).

\section{Problem Setup and Background} \label{sec: background}

\label{sec: background MM}
Model merging aims to aggregate $N$ existing models which are fine-tuned from a large pre-trained model $\boldsymbol{W}_0$ on different downstream datasets. Each fine-tuned task-specific model has parameters $\boldsymbol{W}_i$ which are learned from the local dataset $D_i$.~To extract task-specific information for task $i \in [N]$,~\citet{ilharco2022editing} introduced the task vector $\boldsymbol{\theta}_i = \boldsymbol{W}_i-\boldsymbol{W}_0$.~The task vector encodes task-specific information and allows the analysis of individual task's characteristics.~The goal of model merging is to find merged parameters $\boldsymbol{W}_m$ that performs well on all tasks by designing an aggregation algorithm $A$ to combine the task vectors:
\begin{eqnarray*}
    \boldsymbol{W}_m &=& A\big(\boldsymbol{W}_0, \boldsymbol{\theta}_1,\dots,\boldsymbol{\theta}_N\big) \ .
\end{eqnarray*}
\cite{ilharco2022editing} show that we can obtain multi-task models $\boldsymbol{W}_m$ by performing simple task arithmetic operations on the task vectors $\theta_1$, $\theta_2$, \ldots, $\theta_N$:
\begin{eqnarray} \label{eq: task arith}
    \boldsymbol{W}_m &=& \boldsymbol{W}_0 \ +\ \lambda\sum_{i=1}^N \boldsymbol{\theta}_i \ ,
\end{eqnarray}
where $\lambda$ is a scaling coefficient, usually tuned on a validation set.~When $\lambda=1/N$, Eq.~\eqref{eq: task arith} reduces to performing weight averaging over $\boldsymbol{W}_i$~\citep{wortsman2022robust}.~\citet{matena2022merging} further generalize this idea by scaling local parameters with their corresponding Fisher Information Matrix $\boldsymbol{F}_i$ which results in the following Fisher Averaging:
\begin{eqnarray} \label{eq: Fisher}
    \boldsymbol{W}_m &=& \sum_{i=1}^N \alpha_i \left(\sum_{t=1}^N \alpha_t \boldsymbol{F}_t\right)^{-1} \boldsymbol{F}_i\boldsymbol{\theta}_i \ ,
\end{eqnarray}
where 
\begin{eqnarray}
    \mathbf{F}_i &=& \mathbb{E}_{x\sim D_i} \Big[\mathbb{E}_{y\sim p_{\mathbf{W}_i}(y|x)}\Big]\mathbf{G}_i\mathbf{G}_i^\top
\end{eqnarray} with $\mathbf{G}_i=\nabla_{\mathbf{W}_i}\log p_{\mathbf{W}_i}(y|x)$.

Alternatively, RegMean~\citep{jin2022dataless} matches model behavior by aligning activations between the merged and task-specific models at each linear layer.~This leads to a linear regression problem defined by the task-specific data matrix \(\boldsymbol{X}_i\) which introduces another merging rule:
\begin{eqnarray}
\hspace{-1mm}\boldsymbol{W}_m
    =
    \left(
    \sum_{i=1}^N
    \frac{1}{N_i}
    \boldsymbol{X}_i^\top \boldsymbol{X}_i
    \right)^{-1}
    \left(
    \sum_{i=1}^N
    \frac{1}{N_i}
    \boldsymbol{X}_i^\top \boldsymbol{X}_i
    \boldsymbol{\theta}_i
    \right).
\end{eqnarray}

Overall, existing merging methods differ in how they align local solution spaces.~Yet, they essentially view task updates as deterministic quantities to be aggregated without weighing their statistical usefulness across different tasks (Section~\ref{sec:intro}).~We therefore seek a probabilistic treatment of fine-tuning modules that not only account for the local update geometries, but also estimate how confidently each direction should influence the merged module.

\section{A Product-of-Experts (PoE) Perspective on Model Merging} \label{sec: method}

In this section, we develop a new perspective for model merging in which task-specific fine-tuning modules can be interpreted as observations of a shared latent parameter that captures information common across tasks.~The merged model can be viewed as a sample drawn from a product-of-experts (PoE) model combining these local energy-based densities (Section~\ref{sec: merging as EBM}).~In this view, model merging can then be formulated as MAP inference which also provides a unified probabilistic interpretation of existing merging rules. In particular, we show that weight averaging, Fisher-weighted averaging, and RegMean arise as special cases under specific choices of task-specific energy functions, thereby making their implicit distributional assumptions explicit (Section~\ref{sec: energy function}).~Section~\ref{sec:design} then shows that these implicit assumptions are mismatched with the empirical deviations observed between candidate merged updates and task-specific modules, and introduces a redesign of the local energy functions to mitigate this mismatch.

\subsection{POE formulation for model merging} \label{sec: merging as EBM}

Energy-based models (EBMs) \citep{teh2003energy, lecun2006tutorial, song2021train} define probability distributions through an unnormalized function called an \emph{energy} $E(\boldsymbol{\zeta})$ assigning lower energy to more likely input $\boldsymbol{\zeta}$. The resulting distribution is given by
\begin{eqnarray}
p(\boldsymbol{\zeta}) &=& Z^{-1} \cdot \exp{\big(-E(\boldsymbol{\zeta})\big)} \ ,
\end{eqnarray}
where $Z \triangleq \int\exp{(-E(\boldsymbol{\zeta}))}d\boldsymbol{\zeta}$ is the partition function, ensuring the distribution integrates to 1. Under this view, inputs with lower energies correspond to higher probability. This formulation only needs to specify compatibility via the energy without explicit normalization.

Now, considering the problem of merging $N$ task-specific models.~We assume that there exists an unknown latent parameter ${\boldsymbol{\zeta}}$ that is shared across tasks.~Each task-specific fine-tuning module can then be viewed as providing noisy evidence about the latent shared parameter $\boldsymbol{\zeta}$. Therefore, we associate each task with an expert distribution over $\boldsymbol{\zeta}$, where higher probability corresponds to greater compatibility between the latent parameter and the task-specific update. We model this compatibility through an energy function, resulting in the following energy-based expert:
\begin{eqnarray}
p_i(\boldsymbol{\zeta}) &\propto& \exp{\Big(-E_i(\boldsymbol{\zeta})\Big)}\ ,
\end{eqnarray}
where $E_i(\boldsymbol{\zeta})$ is an energy function determined by the task-specific parameter $\boldsymbol{\theta}_i$. The energy function is general and can be designed with specific desiderata. Local energy functions can also be combined naturally through multiplication as established in~\citep{hinton2002training}. Essentially, each task defines its energy $E_i(\boldsymbol{\zeta})$ and under conditional independence assumption of tasks on $\boldsymbol{\zeta}$, the PoE distribution is given as 
\begin{eqnarray} \label{eq: PoE posterior}
p(\boldsymbol{\zeta}) &\propto& \prod_{i=1}^N p_i(\boldsymbol{\zeta}) \ \propto \ \exp{\left( - \sum_{i=1}^N E_i(\boldsymbol{\zeta})\right)} \ .
\end{eqnarray}
This results in the global energy
$E(\boldsymbol{\zeta}) \triangleq \sum_{i=1}^N E_i(\boldsymbol{\zeta}),$
such that $p(\boldsymbol{\zeta}) \propto \exp{(-E(\boldsymbol{\zeta}))}$.~Combining task experts then corresponds to summing up their energies.~The resulting PoE distribution thus concentrates probability mass on configurations $\boldsymbol{\zeta}$ that simultaneously achieve low energy for all experts. Under this view, model merging reduces to inference with an energy-based model over the shared parameter. 

A natural way to obtain a merged model is to compute the maximum-a-posteriori (MAP) estimate:
\begin{eqnarray} \label{eq: general objective}
    \boldsymbol{\zeta}^* &=& \arg\max_{\boldsymbol{\zeta}} p(\boldsymbol{\zeta}) \ = \ \arg\min_{\boldsymbol{\zeta}} E(\boldsymbol{\zeta})\nonumber\\
    &=& \arg\min_{\boldsymbol{\zeta}} \sum_{i=1}^N E_i(\boldsymbol{\zeta}) \ .
\end{eqnarray}
Interestingly, this formulation unifies existing averaging heuristics as cases of quadratic energies as shown in Section \ref{sec: energy function}.~At the same time, it also allows for new designs of more flexible and robust choices of energy function.~Merging operation can thus be tightly coupled with the design of local energy functions as shown in Eq.~\eqref{eq: general objective}.~The MAP solution of the resulting PoE can then be found via minimizing Eq.~\eqref{eq: general objective}, resulting in a merged model corresponding to the most probable shared parameter.

\subsection{Gaussian Experts Recover Existing Merging Rules} \label{sec: energy function}
The key modeling aspect in the PoE-EBM framework is the design choice of the task-specific function $E_i(\boldsymbol{\zeta})$. Different choices encode different assumptions about how task-specific parameters deviate from the shared latent structure.~These choices also lead to various merging designs.~Several of which were rediscovered below as special cases of PoE with Gaussian experts.

{\bf A.~Gaussian Expert Formulation.}~Assume each task-specific parameter $\boldsymbol{\theta}_i$ is generated from the shared latent parameter $\boldsymbol{\zeta}$ under Gaussian noise:
\begin{eqnarray} \label{eq: Gaussian model}
    \boldsymbol{\theta}_i &=& \boldsymbol{\zeta} + \boldsymbol{\epsilon}_i, \ \ \ \  \boldsymbol{\epsilon}_i \sim \mathcal{N}(\boldsymbol{0}, \boldsymbol{\Sigma}_i) \ .
\end{eqnarray}
This induces the likelihood over the latent $\boldsymbol{\zeta}$,
\begin{eqnarray}
    p_i(\boldsymbol{\zeta}) &\propto& \exp{\Big(-\frac{1}{2}(\boldsymbol{\zeta}-\boldsymbol{\theta}_i)^\top \boldsymbol{\Sigma}_i^{-1}(\boldsymbol{\zeta}-\boldsymbol{\theta}_i)}\Big) \ .
\end{eqnarray}
Thus, the corresponding energy function is quadratic
\begin{eqnarray} \label{eq: Gaussian energy}
    E_i(\boldsymbol{\zeta}) &=& \frac{1}{2}(\boldsymbol{\zeta}-\boldsymbol{\theta}_i)^\top \boldsymbol{\Sigma}_i^{-1}(\boldsymbol{\zeta}-\boldsymbol{\theta}_i) \ .
\end{eqnarray}
The MAP solution in Eq.~\eqref{eq: general objective} thus exhibits a closed form:
\begin{eqnarray} \label{eq: MAP Gauss}
    \boldsymbol{\zeta}^* &=& \Bigg(\sum_{i=1}^N \boldsymbol{\Sigma}_i^{-1} \Bigg)^{-1} \Bigg(\sum_{i=1}^N \boldsymbol{\Sigma}_i^{-1} \boldsymbol{\theta}_i\Bigg) \ .
\end{eqnarray}

{\bf B.~Recovering Existing Merging Methods.}~We now show that various classical merging methods are special cases of the above PoE with Gaussian experts which exhibit quadratic local energy functions.~In particular, we will show that the above MAP estimator admits several well-known model merging rules as special cases under different choices and structural assumption of the precision matrices $\boldsymbol{\Sigma}_i^{-1}$. 

{\bf 1.~Uniform Averaging.}~Suppose all precision matrices are isotropic,
$\boldsymbol{\Sigma}_i^{-1} = \boldsymbol{I}$, Eq.~\eqref{eq: MAP Gauss} simplifies to
\begin{eqnarray}
\boldsymbol{\zeta}^{\mathrm{avg}}
&=&
\frac{1}{N}
\sum_{i=1}^N
\boldsymbol{\theta}_i\ ,
\end{eqnarray}
which recovers uniform averaging~\citep{wortsman2022robust}.

{\bf 2.~Fisher-Weighted Averaging.}~Suppose the precision matrix of each task is chosen as its Fisher information matrix,
$\boldsymbol{\Sigma}_i^{-1} = \boldsymbol{F}_i$,
Eq.~\eqref{eq: MAP Gauss} becomes
\begin{eqnarray}
\boldsymbol{\zeta}^{\mathrm{FA}}
&=&
\left(
\sum_{i=1}^N \boldsymbol{F}_i
\right)^{-1}
\left(
\sum_{i=1}^N \boldsymbol{F}_i \boldsymbol{\theta}_i
\right) \ ,
\end{eqnarray}
which recovers Fisher averaging~\citep{matena2022merging}.

{\bf 3.~RegMean.}~When the precision matrix is taken as the Gram matrix of empirical data,
\begin{eqnarray}
\boldsymbol{\Sigma}_i^{-1}
&=&
\frac{1}{N_i}
\boldsymbol{X}_i^\top \boldsymbol{X}_i \ ,
\end{eqnarray}
the MAP estimator reduces to
\begin{eqnarray}
\boldsymbol{\zeta}^{\mathrm{RM}}
=
\left(
\sum_{i=1}^N
\frac{1}{N_i}
\boldsymbol{X}_i^\top \boldsymbol{X}_i
\right)^{-1}
\left(
\sum_{i=1}^N
\frac{1}{N_i}
\boldsymbol{X}_i^\top \boldsymbol{X}_i
\boldsymbol{\theta}_i
\right),
\end{eqnarray}
which recovers RegMean
\citep{jin2022dataless}.

\section{From Gaussian to Heavy-Tailed Expert Models}
\label{sec:design}

Section~\ref{sec: method} shows that several existing merging methods can be recovered as PoE models with Gaussian experts.~We now show that these Gaussian formulations impose light-tailed assumptions on directional residuals, which mismatch the heavy-tailed behavior observed in practice (Section~\ref{sec:limit}).~To address this mismatch, we introduce a heavy-tailed PoE formulation based on Cauchy experts (Section~\ref{sec:cauchy}).

\subsection{Limitations of Gaussian Experts}
\label{sec:limit}
As shown above, the use of quadratic energy fields lead to EBMs with Gaussian shapes.~Such Gaussian experts are closely related to the framework of \citep{daheim2023model} which assumes a Gaussian prior over task-specific parameters.~This is followed by a Laplace approximation to obtain the Gaussian posterior for the merged parameters.~Under specific choices of the precision matrix, their formulation reduces to precision-weighted averaging and recovers standard schemes such as weight averaging and Fisher-weighted averaging.~However, such Gaussian structures often do not sufficiently capture the tail behaviors of the merged models due to their fast-decaying tails.~This represents a structural mismatch according to our empirical findings in Fig.~\ref{fig:heavy_tail} which shows heavy tail behavior of the merged models.

This can be seen by analyzing the {\bf directional residual} of the merged model with respect to each task-specific module. Given a candidate merged parameter $\boldsymbol{\zeta}$ and the $i$-th fine-tuning module $\boldsymbol{\theta}_i$, the directional residual is defined as
\begin{eqnarray}
    r_i(\boldsymbol{\zeta}) &\triangleq& (\boldsymbol{\zeta}-\boldsymbol{\theta}_i)^\top \boldsymbol{\theta}_i \ . \label{eq:dir-res}
\end{eqnarray}
Intuitively, this computes the difference $ (\boldsymbol{\zeta} - \boldsymbol{\theta}_i)$ between the merged and individual models which is then projected onto the direction of the task update $\boldsymbol{\theta}_i$.~The result measures how far the merged solution drifted from the task-specific module along that task's preferred direction.~Under a Gaussian model $\boldsymbol{\zeta} - \boldsymbol{\theta}_i \sim \mathcal{N}(\boldsymbol{0}, \boldsymbol{\Sigma}_i)$, each directional residual thus follows a Gaussian $r_i(\boldsymbol{\zeta}) \sim \mathcal{N}({0}, \boldsymbol{\theta}_i^\top \boldsymbol{\Sigma}_i\boldsymbol{\theta}_i)$.

As a result, existing choices of Gaussian experts, and by extension the classical merging methods they recover, implicitly impose light-tailed assumptions on the residuals.~However, a closer inspection of the directional residuals of different merging methods reveals an intrinsic heavy-tailed behavior that stands in contrast to these Gaussian assumptions.~In Fig.~\ref{fig:heavy_tail}, we plot the empirical density of the directional residuals across all tasks and layers when merging 7 (fully) fine-tuned ViT-L-14 models and overlay Gaussian and heavy-tailed Cauchy distributions for comparison.~It shows a heavy-tailed behavior: the decay in the tail region indicates a heavy-tailed distribution, deviating from Gaussian behavior and following more closely with the Cauchy distribution. In particular, large residuals occur more frequently than would be predicted under a Gaussian assumption. 

\subsection{Heavy-Tailed Expert Models}
\label{sec:cauchy}
To capture this inherent heavy-tailed residual distribution, we adopt a robust energy function whose logarithmic growth reduces the influence of task-specific updates with large directional residuals. The negative log-density of a Cauchy distribution \citep{liu2014robustness} naturally exhibits this behavior, motivating the following Cauchy expert design:
\begin{eqnarray}
    E^{Cauchy}_i(\boldsymbol{\zeta}) &\triangleq& \log\left(1+ \frac{r_i(\boldsymbol{\zeta})^2}{\gamma^2}\right)\ ,
\end{eqnarray}
where $r_i(\boldsymbol{\zeta})$ denotes the directional residual at task-specific parameter $\boldsymbol{\theta}_i$ and $\gamma>0$ is a user-defined scale controlling the tail heaviness of the distribution. This energy function induces the following (unnormalized) Cauchy density:
\begin{eqnarray} \label{eq:energy cauchy}
\hspace{-8mm}    p_i(\boldsymbol{\zeta}) &\propto& \exp{\Big(- E_i^{Cauchy}(\boldsymbol{\zeta})\Big)} 
    \ =  \frac{\gamma^2}{\gamma^2 + r_i(\boldsymbol{\zeta})^2} \ ,
\end{eqnarray}
which is heavy-tailed. Under conditional independence assumption of tasks, the PoE-EBM posterior factorizes according to Eq.~\eqref{eq: PoE posterior}, leading to the following global energy:
\begin{eqnarray} \label{eq: global Cauchy energy}
\hspace{-15mm}    E^{Cauchy}(\boldsymbol{\zeta}) &=& \sum_i \log\Big(\gamma^2 + r_i(\boldsymbol{\zeta})^2\Big) + \mathcal{C} \ ,
\end{eqnarray}
where $\mathcal{C}$ is a constant independent of $\boldsymbol{\zeta}$. This defines the Cauchy score as negative of the global energy gradient:
\begin{eqnarray} \label{eq: score cauchy}
\hspace{-15mm}    S^{Cauchy}(\boldsymbol{\zeta}) &\triangleq& - \nabla_{\boldsymbol{\zeta}} E^{Cauchy}\nonumber\\
    &=&  -\sum_{i=1}^N \frac{2r_i(\boldsymbol{\zeta})}{\gamma^2 + r_i(\boldsymbol{\zeta})^2}\nabla_{\boldsymbol{\zeta}}r_i(\boldsymbol{\zeta}) \ ,
\end{eqnarray}
where the residual gradient is $\nabla_{\boldsymbol{\zeta}}r_i(\boldsymbol{\zeta}) = \boldsymbol{\theta}_i$. We further provide insights into the Cauchy score and the connection between Cauchy and Gaussian experts in Appendix \ref{sec: Cauchy insights}.

\section{MAP Inference Algorithm and Convergence Guarantee} \label{sec: merging algo}

We will now develop a practical MAP inference algorithm for the previously established heavy-tailed PoE model.~We first derive a fixed-point characterization of the optimal merged update and then use it to obtain an iterative procedure with a convergence guarantee.

As shown in the general PoE-EBM setting in Section \ref{sec: merging as EBM}, merging corresponds to computing the MAP estimator in Eq.~\eqref{eq: general objective}.~For Cauchy experts with global Cauchy energy in Eq.~\eqref{eq: global Cauchy energy}, the inference task then becomes
\begin{eqnarray} \label{eq: Cauchy MAP obj}
\boldsymbol{\zeta}^* &=& \arg\min \sum_{i=1}^N \log\Big(\gamma^2 + r_i(\boldsymbol{\zeta})^2\Big) \ .
\end{eqnarray}
Unlike the quadratic case in Eq.~\eqref{eq: MAP Gauss}, this optimization loss is nonconvex and does not admit a closed-form minimizer. However, we can exploit its structure to derive an explicit optimality characterization. In particular, the MAP estimator can be expressed as the closed-form solution of a nonlinear equation, revealing that robust model merging amounts to a residual-dependent weighted consensus among task-specific models. The following theorem formalizes this closed-form characterization of the MAP solution.

\begin{theorem}[Closed-form characterization of MAP] \label{thm: 1}
Let $\{\boldsymbol{\theta}_i\}_{i=1}^N$ be a set of task-specific fine-tuning modules, and define the following auxiliary function:
\begin{eqnarray}
\hspace{-10mm}u_i(\boldsymbol{\zeta})
\hspace{-2mm}&=&\hspace{-2mm}
\left(\big[(\boldsymbol{\zeta} - \boldsymbol{\theta}_i)^\top \boldsymbol{\theta}_i\big]^2 + \gamma^2\right)^{-1} \ \ \text{with}\ \  \gamma > 0 \ .
\end{eqnarray}
It then follows that any stationary point $\boldsymbol{\zeta}^*$ of the MAP loss in Eq.~\eqref{eq: Cauchy MAP obj} satisfies $\boldsymbol{\zeta}^* = F(\boldsymbol{\zeta}^*)$
with a closed-form mapping $F$ defined below;
\begin{eqnarray} \label{eq: F map}
\hspace{-1.5mm}F(\cdot) \hspace{-2mm}&\triangleq&\hspace{-2mm} \left(
    \sum_{i=1}^N
    u_i(\cdot)\,
    \boldsymbol{\theta}_i \boldsymbol{\theta}_i^\top + \eta \boldsymbol{I}
    \right)^{-1}
    \left(
    \sum_{i=1}^N
    u_i(\cdot)\,
    \boldsymbol{\theta}_i \boldsymbol{\theta}_i^\top \boldsymbol{\theta}_i
    \right)\nonumber \\
    &\triangleq& \boldsymbol{H}(\cdot)^{-1} \boldsymbol{b}(\cdot) \ ,
\end{eqnarray}
\end{theorem}
where we define $\boldsymbol{H}(\cdot) \triangleq \sum_{i=1}^N
u_i(\cdot)\, \boldsymbol{\theta}_i \boldsymbol{\theta}_i^\top + \eta \boldsymbol{I}$ and $\boldsymbol{b}(\cdot) \triangleq \sum_{i=1}^N u_i(\cdot)\, \boldsymbol{\theta}_i \boldsymbol{\theta}_i^\top \boldsymbol{\theta}_i$ for ease of notation.~Here, $\eta > 0$ is a conditioning hyper-parameter to ensure the inversion operator in the definition of $F$ is well-defined.

Under mild boundedness assumptions on the fine-tuning modules (see Assumption~\ref{ass: bounded tv}), Theorem~\ref{theo: contract} further shows that $F$ is a contractive map with a unique fixed point solution.~Consequently, repeated application of $F$ converges to this fixed point, which corresponds to the MAP estimate (i.e., the optimal merged solution) of \texttt{PoE-EBM}.~This forms the basis of our merging algorithm (see Algorithm~\ref{alg:probabilistic-merging}).
    
\begin{assumption}[Bounded fine-tuning modules] \label{ass: bounded tv}
There exists a constant \(M>0\) such that $\|\boldsymbol{\theta}_i\|_2 \le M, \ \forall i\in [N].$
\end{assumption}

We note that in fine-tuning regimes, task-specific parameters are fine-tuned from a common pre-trained model with common practices such as $\ell_2$ regularization with small learning rates or constraining the rank of the fine-tuning modules. Consequently, these modules often remain within a bounded region of the parameter space.

To validate this assumption, we compute the $\ell_2$ norm of full fine-tuned task vectors across vision tasks and their ratios relative to the pretrained model. As shown in Table \ref{tab:l2_ratio}, the ratio is consistently in the range of $0.6\%-0.8\%$ across all tasks and models. Motivated by this empirical observation, we restrict our analysis to a local neighborhood around the fixed point $\boldsymbol{\zeta}^*$ and assume that all iterates remain within a unit ball centered at $\boldsymbol{\zeta}^*$ as formally stated below.

\begin{assumption}[Local Contraction Neighborhood] \label{ass: local}
Let $\boldsymbol{\zeta}^\ast$ be a fixed point of $F$.~The iterates $\{\boldsymbol{\zeta}^{(k)}\}_k$ generated by $F$ are assumed to remain in a closed unit ball around $\boldsymbol{\zeta}^\ast$:
\begin{eqnarray}
\hspace{-9mm}\forall k\ :\ \boldsymbol{\zeta}^{(k)} \in B_1(\boldsymbol{\zeta}^\ast) \hspace{-1mm}&\triangleq&\hspace{-1mm} \big\{ \boldsymbol{\zeta} \;:\; \|\boldsymbol{\zeta} - \boldsymbol{\zeta}^\ast\| \ \ \le\ \ 1 \big\} \ .
\end{eqnarray}
\end{assumption}

Given the above assumption, we can now show that $F$ is contractive via Theorem~\ref{theo: contract} below.

\begin{theorem}[Contraction of $F$]
\label{theo: contract}
Under Assumptions~\ref{ass: bounded tv} and~\ref{ass: local}, it follows that $F$ is a contractive mapping where $\| F(\boldsymbol{\zeta}^\ast) - \boldsymbol{\zeta}^{(k+1)}\| = \| F(\boldsymbol{\zeta}^\ast) - F(\boldsymbol{\zeta})^{(k)}\| \leq L\|\boldsymbol{\zeta}^\ast -\boldsymbol{\zeta}^{(k)}\|$ with $L$ being a provably small Lipschitz constant,
\begin{eqnarray}
\hspace{-26mm}L \hspace{-1mm}&\triangleq&\hspace{-1mm} 2M \sum_{i=1}^N \left\|\boldsymbol{H}^{-1}\boldsymbol{J}_i(\boldsymbol{\zeta}^*)\right\| \ \ \le \ \ 1 \ ,
\end{eqnarray}
where $\boldsymbol{H}$ is previously defined in Theorem~\ref{thm: 1} and  $\boldsymbol{J}_i(\boldsymbol{\zeta}^*) \triangleq u_i(\boldsymbol{\zeta}^*)\Big(\boldsymbol{\theta}_i\boldsymbol{\theta}_i^\top\boldsymbol{\theta}_i - \boldsymbol{\theta}_i\boldsymbol{\theta}_i^\top F(\boldsymbol{\boldsymbol{\zeta}^*})\Big)$.
\end{theorem}
Due to limited space, the proof of Theorem~\ref{theo: contract} is deferred to Appendix~\ref{sec: proof contract}.~We also show empirically in Fig.~\ref{fig:Lips_const} that $L<1$ in all experiments, asserting that $F$ is indeed a contracting map in a neighborhood of the true MAP $\boldsymbol{\zeta}^*$.~This leads to a practical merging procedure via iterating $F$ in Algorithm~\ref{alg:probabilistic-merging}.~Its complexity analysis is provided in Appendix~\ref{sec: comp analysis}.

\begin{algorithm}[t]
\caption{\texttt{PoE-EBM} merging}
\label{alg:probabilistic-merging}
\begin{algorithmic}[1]
\STATE {\bfseries INPUT:}~Backbone model $\boldsymbol{W}_0$, fine-tuning modules $\{\boldsymbol{\theta}_i\}_{i=1}^N$, conditioning parameter $\eta > 0$, scaling coefficient $\lambda$, and number of iterations $T$.
\STATE {\bfseries OUTPUT:} Merged module ${\boldsymbol{\zeta}}$, merged model $\boldsymbol{W}_m$.

\STATE Initialize ${\boldsymbol{\zeta}}^{(0)}$

\FOR{$k = 0$ to $T-1$}
    \FOR{$i = 1$ to $N$}
        \STATE Compute residual
        $
        r_i^{(k)} \leftarrow (\boldsymbol{\zeta}^{(k)} - \boldsymbol{\theta}_i)^\top \boldsymbol{\theta}_i
        $
        \STATE Compute
        $
        u_i^{(k)} \leftarrow 1/((r_i^{(k)})^2 + \gamma \|\boldsymbol{\theta}_i\|^2)
        $
    \ENDFOR
    \STATE Compute $\boldsymbol{H}^{(k)} = \boldsymbol{H}({\boldsymbol{\zeta}}^{(k)}), \ \boldsymbol{b}^{(k)} = \boldsymbol{b}({\boldsymbol{\zeta}}^{(k)})$ via \ref{eq: F map}
    \STATE Compute new iterate ${\boldsymbol{\zeta}}^{(k+1)} = \big(\boldsymbol{H}^{(k)}\big)^{-1} \boldsymbol{b}^{(k)}$
\ENDFOR
\STATE \textbf{return} ${\boldsymbol{\zeta}} \triangleq ({\boldsymbol{\zeta}}^{(T)})$, \ \  $\boldsymbol{W}_m = \boldsymbol{W}_0 + \lambda \boldsymbol{\zeta}^{(T)}$
\end{algorithmic}
\end{algorithm}

\section{Experiments and Results} \label{sec: Experiments}
In this section, we validate the effectiveness our framework on diverse empirical settings covering both vision and language tasks. We describe the setup of our experiments in Section~\ref{sec: exp setup} and provide the main results in Section~\ref{sec: main exp}. We also provide additional ablation analysis in Section~\ref{sec: ablation}. 

For clarity, we use the following highlighting convention: (1) best accuracy is \textbf{bolded}; (2) second best accuracy is \underline{underlined}; and (3) task-specific fine-tuning accuracy (as performance upper bound) is colored \textcolor{blue}{blue}.

\begin{table*}[t]
\centering
\caption{Multi-task performance comparison when merging ViT-B/32 (fully finetuned) across 7 vision benchmarks (absolute accuracy). Task-specific finetuning accuracy is \textcolor{blue}{blue} - performance accuracy upper bound.}
\label{tab:vitb32_results}
\resizebox{\textwidth}{!}{
\begin{tabular}{|l|l l l l l l l l|}
\hline
& \multicolumn{8}{c|}{\textbf{ViT-B/32}} \\
\cline{2-9}
\textbf{Method} 
& \textbf{MNIST}
& \textbf{SVHN}
& \textbf{Cars}
& \textbf{DTD}
& \textbf{GTSRB}
& \textbf{EuroSAT}
& \textbf{RESISC45}
& \textbf{Average} \\
\hline

\textbf{Finetuning} 
& \textcolor{blue}{99.67}
& \textcolor{blue}{97.46}
& \textcolor{blue}{76.36}
& \textcolor{blue}{97.29}
& \textcolor{blue}{99.11}
& \textcolor{blue}{99.78}
& \textcolor{blue}{95.44}
& \textcolor{blue}{95.02} \\
\hline

Weight averaging
& 85.00
& 64.05
& 60.10
& 51.76
& 55.38
& 63.44
& 68.30
& 64.00 \\

Task Arithmetic (TA)
& 90.65
& 71.66
& 60.17
& 55.00
& 62.69
& 66.15
& 69.22
& 67.93 \\

TA-DARE-TIES
& 95.98
& 82.03
& 59.41
& 60.90
& 70.84
& 70.48
& 67.52
& 72.45 \\

Fisher Merging
& 89.77
& 82.35
& \underline{69.15}
& 54.52
& 57.98
& 77.11
& 74.03
& 72.13 \\

DOGE-TA
& \underline{98.41}
& \underline{87.52}
& \textbf{70.53}
& \underline{64.31}
& \underline{87.76}
& \underline{89.93}
& \textbf{82.37}
& 82.97 \\

Concrete-TA
& 96.99
& 80.01
& 57.99
& 58.56
& 71.27
& 76.22
& 71.59
& 73.23 \\
\hline

\texttt{PoE-EBM} (Ours)
& \textbf{98.88}
& \textbf{91.65}
& 65.93
& \textbf{76.06}
& \textbf{88.80}
& \textbf{90.04}
& \underline{77.92}
& \textbf{84.18} \\

\hline
\end{tabular}
}
\end{table*}

\subsection{Experimental setup} \label{sec: exp setup}
\textbf{Tasks.} We evaluate on both vision and language benchmarks. For vision, we consider a 7-task benchmark and an extended 13-task benchmark. For language task, we use 8 datasets from GLUE benchmark \citep{wang2018glue}.

\textbf{Models.} For vision tasks we merge CLIP ViT-B/32 and ViT-B/14 models \citep{radford2021learning} under both fully fine-tuned and LoRA fine-tuned settings. For language task, we merge LoRA fine-tuned Flan-T5-base and Flan-T5-large.

\textbf{Model Merging Methods.} We compare our \texttt{PoE-EBM} against common model merging baselines including Weight Averaging \citep{wortsman2022robust}, Task Arithmetic \citep{ilharco2022editing}, DARE-TIES \citep{yadav2023ties, yu2024language}. We also compare with recent subspace-based and optimization-based merging methods such as Concrete \citep{tang2023concrete}, KnOTS \citep{stoica2024model} (for LoRA-fine-tuned vision models) and DOGE \citep{wei2025modeling}. 

\textbf{Metrics.} We report absolute accuracy and normalized accuracy (relative to each task’s fine-tuned model). Full experimental details are provided in Appendix \ref{sec: exp details}.

\begin{figure*}[t]
    \centering
    \includegraphics[width=0.7\linewidth]{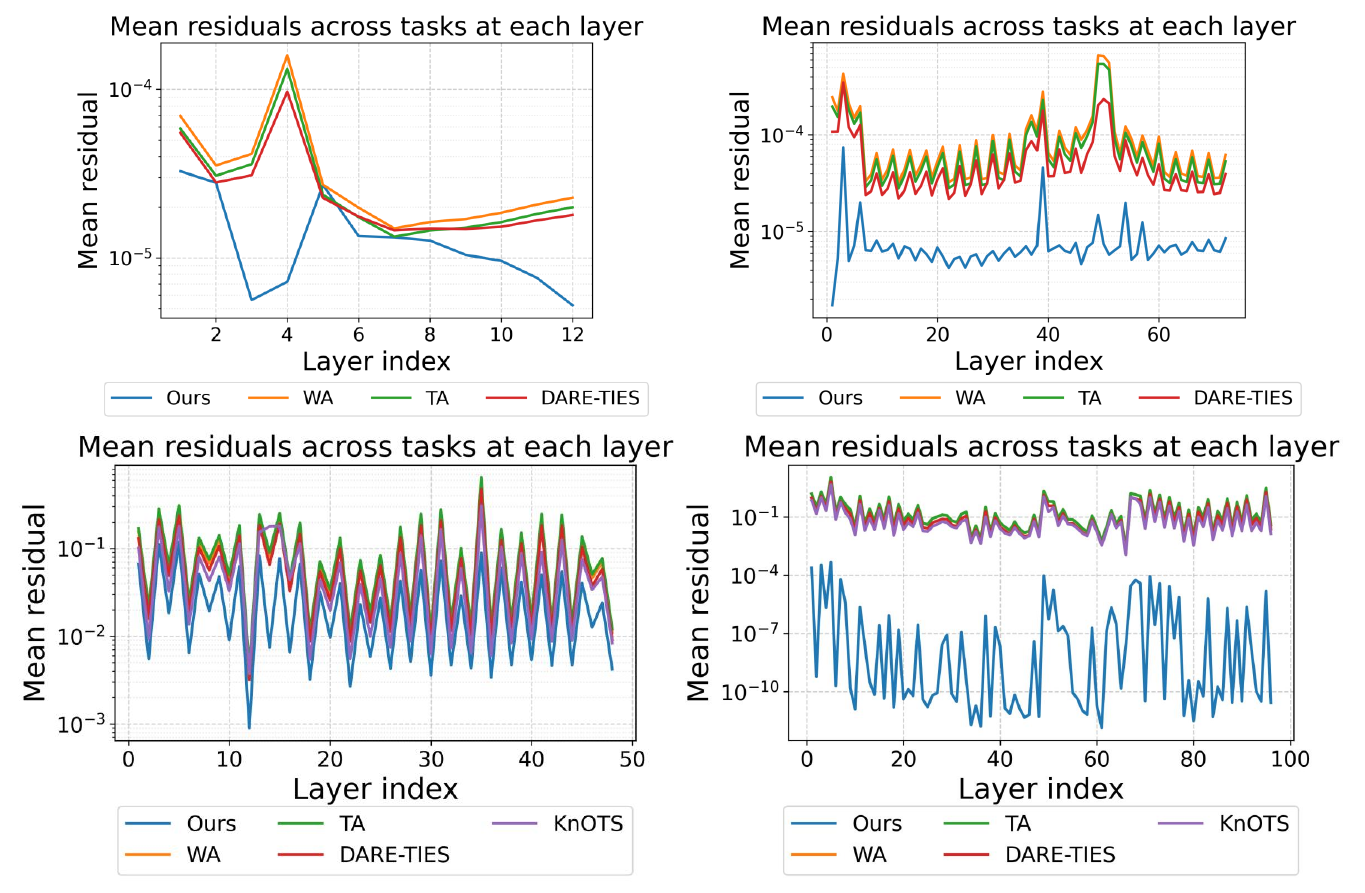}
    \caption{Plots of average directional residual over tasks (log-scale) at every layer weight in ViT-B-32 (left column) and ViT-L-14 (right column) incurred by different merging methods. Top row - merging fully fine-tuned models. Bottom row - merging LoRA fine-tuned models. It can be observed that \texttt{PoE-EBM} consistently achieve lower directional residual values at every layer than other merging methods.}
    \label{fig: residual combined}
    
\end{figure*}

\subsection{Model Merging for Vision Tasks} \label{sec: main exp}
We present the results for merging fully-finetuned ViT-B-32 and ViT-L-14 models on 7 vision benchmarks in Table \ref{tab:vitb32_results} and Table \ref{tab:vitl14_results}, respectively. We also present the results for LoRA-finetuned ViT-L-14 models in Table \ref{tab:vitl14lora_results}. 

\texttt{PoE-EBM} consistently outperforms both simple weight averaging and task arithmetic baselines on every individual task. More recent alignment-based methods, such as Concrete and DOGE, yield further improvements over task arithmetic; however, their average accuracies remain lower to our approach. In general, \texttt{PoE-EBM} consistently achieves best or second best accuracies on individual tasks and best accuracy on average.

We note that while DOGE-TA outperforms \texttt{PoE-EBM} on a subset of tasks, these gains come at the expense of substantial degradation on others. For instance, when merging ViT-B/32 and ViT-L/14 models, DOGE-TA achieves absolute accuracies of $64.31$ and $72.23$ on DTD, respectively, which are markedly lower than those obtained by \texttt{PoE-EBM} ($76.06$ and $83.56$). This behavior highlights that our approach yields a more favorable trade-off across tasks, maintaining strong general performance without disproportionately sacrificing any individual task. \texttt{PoE-EBM} also demonstrates strong robustness as the number of tasks increases. As shown in Table \ref{tab: 13 tasks}, our framework achieves substantial performance gains over all baselines when merging a larger set of tasks, for both ViT-B/32 and ViT-L/14 models.

Table~\ref{tab:vitl14lora_results} shows that under LoRA-fine-tuned settings, our merging framework substantially outperforms task arithmetic and weight averaging variants and achieves better performance than LoRA-specific merging methods such as KnOTS-TIES and KnOTS-DARE-TIES \citep{stoica2024model}. Taken together, these results demonstrate that \texttt{PoE-EBM} consistently generalizes across model scales, numbers of merged tasks, and fine-tuning regimes, without requiring architecture-specific modifications.

\begin{table}[t] 
\raggedright
\caption{Average performance comparison when merging fully finetuned ViT-L/14 models and ViT-L/32 models across 13 vision benchmarks. Task-specific fine-tuning accuracy is \textcolor{blue}{blue} - performance accuracy upper bound.}
\begin{tabular}{|l | l | l|}
\hline
\textbf{Method} & \textbf{ViT-L-14} & \textbf{ViT-B-32} \\
\hline

\textbf{Finetuning} 
& \textcolor{blue}{95.93} & \textcolor{blue}{93.20} \\
\hline

Weight Averaging
& 71.08 & 65.22\\
Task Arithmetic (TA)
& 80.57 & 66.38 \\
TA-DARE-TIES
& 81.18 & 67.09 \\
Fisher Merging
& 79.61 & 69.87 \\
\hline
\texttt{PoE-EBM} (Ours)
& \textbf{87.10} & \textbf{77.96} \\
\hline
\end{tabular}
\label{tab: 13 tasks}
\end{table}

\begin{table*}[t]
\raggedright
\caption{Multi-task performance comparison when merging Flan-T5-base models (LoRA-fine-tuned) across 8 GLUE benchmarks. Task-specific fine-tuning accuracy is colored  \textcolor{blue}{blue} which indicates the upper bound on performance accuracy for model merging.~It can be observed that \texttt{PoE-EBM} achieves the best rank among the considered methods across tasks.}
\label{tab: flan-base}
\begin{tabular}{|l| l l l l l l l l | l l|}
\hline
\textbf{Method}
& \textbf{CoLA}
& \textbf{MNLI}
& \textbf{MRPC}
& \textbf{QNLI}
& \textbf{QQP}
& \textbf{RTE}
& \textbf{SST-2}
& \textbf{STS-B}
& \textbf{Average}
& \textbf{Rank} \\
\hline

\textbf{Finetuning}
& \textcolor{blue}{69.13}
& \textcolor{blue}{82.70}
& \textcolor{blue}{85.50}
& \textcolor{blue}{90.90}
& \textcolor{blue}{84.00}
& \textcolor{blue}{84.40}
& \textcolor{blue}{92.90}
& \textcolor{blue}{87.40}
& \textcolor{blue}{84.62}
& \textcolor{blue}{NA} \\
\hline

Weight Averaging
& \textbf{69.70}
& 59.65
& 78.92
& 90.07
& 83.79
& \underline{80.51}
& 91.12
& 71.89
& 78.21
& \underline{2.75} \\

Task Arithmetic
& \underline{69.32}
& 59.00
& 78.68
& \underline{90.13}
& \underline{83.84}
& 79.06
& 91.51
& 72.87
& 78.05
& 3.00 \\

TIES-Merging
& 69.13
& 59.09
& 78.68
& 90.08
& \textbf{83.91}
& 80.14
& 91.51
& 71.85
& 78.05 
& 3.50 \\

Concrete-TA
& 69.22
& 58.21
& 78.19
& 89.97
& 83.60
& 79.42
& 91.63
& \underline{73.24}
& 77.93
& 3.50 \\

DOGE-TA
& 69.12
& \underline{71.92}
& \underline{80.93}
& \textbf{90.32}
& 83.51
& 79.82
& \underline{92.53}
& 71.13
& \underline{79.91}
& 3.13 \\

\hline

\texttt{PoE-EBM} (Ours)
& 69.22
& \textbf{75.33}
& \textbf{82.84}
& 88.87
& 82.91
& \textbf{80.52}
& \textbf{92.55}
& \textbf{82.77}
& \textbf{81.87}
& \textbf{2.38} \\
\bottomrule
\end{tabular}
\end{table*}

\subsection{Model Merging for Language Tasks}
We now present results on eight datasets from the GLUE benchmark using Flan-T5-base and Flan-T5-large models in Tables \ref{tab: flan-base} and \ref{tab: Flan-large}. As noted in \citep{tang2023concrete}, pretrained text-generation models already exhibit strong inherent multitask capabilities, which can limit the extent of gains achievable through task-specific fine-tuning. Despite this, \texttt{PoE-EBM} achieves the best average generation performance across tasks on both model architectures, indicating that our framework remains effective even in regimes where performance improvements are intrinsically constrained.

\textbf{Scaling to larger LLMs.}~We further evaluate \texttt{PoE-EBM} by merging three 7B-parameter models: Vicuna-7B, Llama-2-Coder, and WizardMath, all fine-tuned from the same Llama-2-7B backbone. The merged model is evaluated on GSM8K math problems. Despite operating on multi-billion-parameter models, \texttt{PoE-EBM} requires only \textbf{122 seconds on 2$\times$A100 GPUs}, demonstrating its practical computational cost. In addition, \texttt{PoE-EBM} achieves the highest accuracy among all compared methods, outperforming both individual expert models and strong merging baselines. As shown in Table~\ref{tab:large_scale}, the merged model surpasses the strongest expert (WizardMath) by more than \textbf{4 percentage points}, indicating that \texttt{PoE-EBM} can effectively combine complementary capabilities from specialized models.

\begin{table}[t]
\raggedright
\caption{Accuracy (\%) on GSM8K dataset after merging three 7B-parameter models (Vicuna-7B, WizardMath, and Llama-2-Coder) fine-tuned from a shared Llama-2-7B backbone. \texttt{PoE-EBM} achieves the highest accuracy, outperforming both individual expert models and standard merging baselines (Task Arithmetic, Weight Averaging).
}
\label{tab:large_scale}
\begin{tabular}{|l|c|}
\hline
\textbf{Method} & \textbf{Accuracy (\%)} \\
\hline
Vicuna-7B & 14 \\
WizardMath & 58 \\
Llama-2-Coder & 10 \\
Task Arithmetic & 46 \\
Weight Averaging & 40 \\
\textbf{PoE-EBM} & \textbf{62} \\
\hline
\end{tabular}
\end{table}

\subsection{Analysis and Ablations} \label{sec: ablation}

We provide ablation studies on the averaged directional residuals incurred by \texttt{PoE-EBM} and other merging baselines, the empirical convergence of the fixed point map and runtime analysis of \texttt{PoE-EBM}.~Additional ablation on the performance sensitivity with respect to scaling parameter $\gamma$ is provided in Appendix~\ref{sec: additional results}.

\begin{figure*}[t] 
    \centering
    \includegraphics[width=0.8\linewidth]{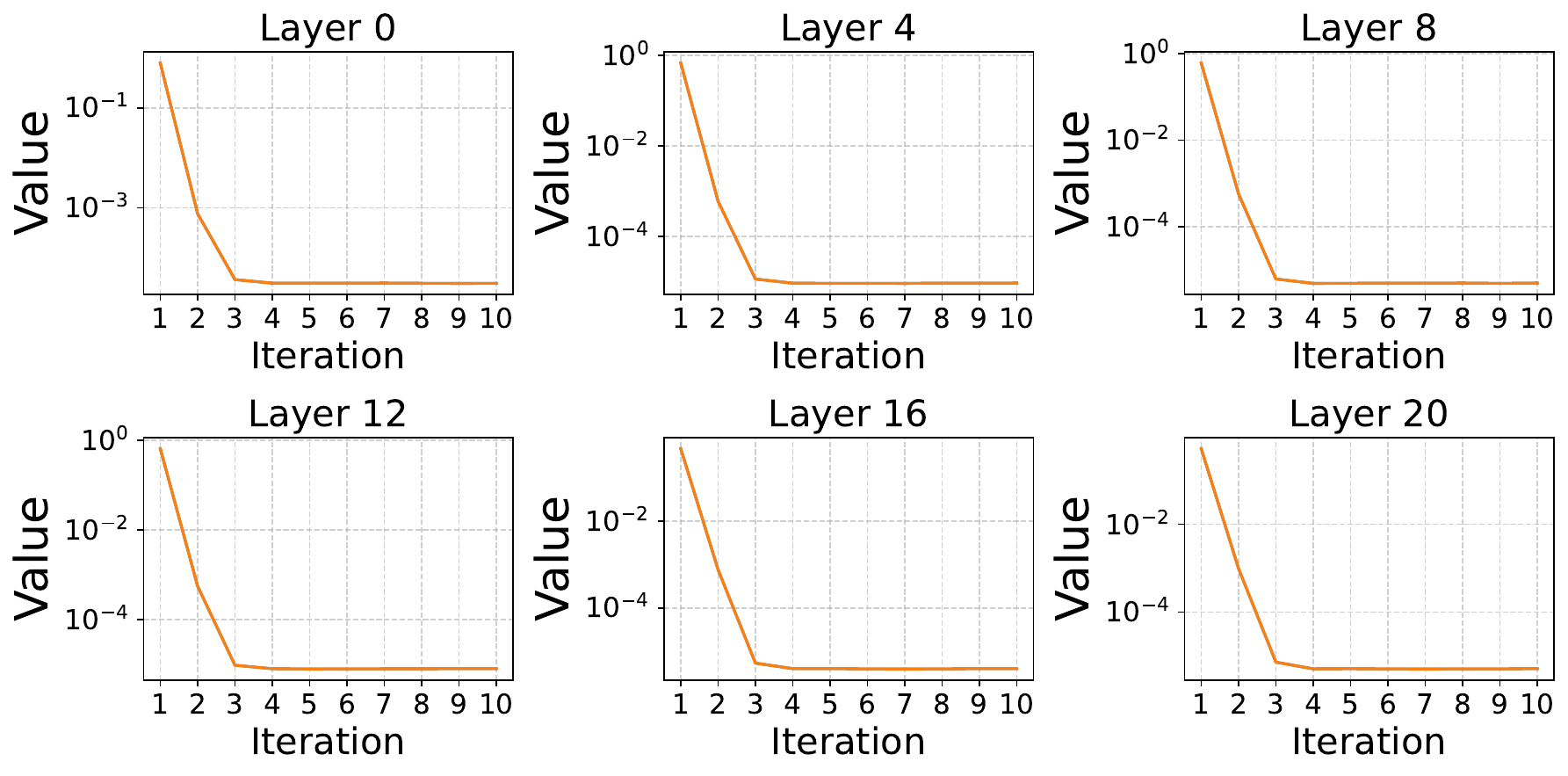}
    \caption{Convergence of the fixed point iteration method (see Algorithm~\ref{alg:probabilistic-merging}) across layers of the merged ViT-L-14 model on 7 vision benchmarks using our \texttt{PoE-EBM} framework with heavy-tailed experts. We visualize $||\boldsymbol{\zeta}^{(k+1)} - \boldsymbol{\zeta}^{(k)}||_2$ across 10 iterations. Our results show rapid convergence to the MAP solution.}
    \label{fig: convergence}
\end{figure*}

\textbf{Residual values across layers.}
To compare directional alignment of different merging methods, we plot the task-average squared residual value 
$$
Mean\Big(||(\boldsymbol{\zeta}^* - \boldsymbol{\theta}_i)^\top \boldsymbol{\theta}_i||^2\Big) \ ,
$$
where $\boldsymbol{\zeta}^*$ is the merged model at each layer of both fully and LoRA fine-tuned ViT-B/32 and ViT-L/14 models, and $Mean()$ represents the average taken across tasks. The results are shown in Figure \ref{fig: residual combined}, respectively. Across all settings, \texttt{PoE-EBM} consistently achieves lower residual magnitude, indicating better directional alignment than other approaches. Notably, other methods exhibit pronounced outlier residuals at certain layers. For example, it can be observed that the residual magnitude is abnormally large at layer 5 of the LoRA-fine-tuned ViT-L/14 in Figure \ref{fig: residual combined}. In contrast, \texttt{PoE-EBM} maintains low residual values across all layers, reflecting more stable and robust alignment behavior.

\textbf{Empirical convergence of the fixed point map.} We empirically validate the convergence of the fixed point iteration algorithm (see Algorithm~\ref{alg:probabilistic-merging}) to compute the MAP point of our \texttt{PoE-EBM}. We report the merging results with ViT-L-14 models on 7 vision tasks. As shown in Figure~\ref{fig: convergence} in Appendix~\ref{sec: additional results}, the sequence $\{\boldsymbol{\zeta}^{(k)}\}$ converges rapidly across all examined layers, with the update magnitude $||\boldsymbol{\zeta}^{(k+1)} - \boldsymbol{\zeta}^{(k)}||_2$ decreasing sharply within the first few iterations and stabilizing thereafter. This convergence pattern indicates that the contracting map admits a solution and the resulting estimate corresponds to the MAP solution under our probabilistic formulation. Importantly, the convergence behavior is uniform across layers, implying that the optimization landscape induced by our model is well-conditioned in practice and that the algorithm is reliable for large-scale model merging. The convergence behavior is observed across layers, implying that the optimization landscape induced by \texttt{PoE-EBM} is well-conditioned in practice.

\textbf{Runtime analysis.} We also evaluate the efficiency of \texttt{PoE-EBM} by measuring the wall-clock runtime of the merging procedure across different tasks and model scales. As reported in Table~\ref{tab:merging-time}, the proposed merging procedure completes in under one minute for all models considered. For the largest model, Flan-T5-large with 0.7B parameters, the full-layer merging process takes only $59.19\mathrm{s}$.~The detailed complexity analysis is deferred to Appendix~\ref{sec: comp analysis}.

\vspace{-5mm}

\section{Conclusion} \label{sec: conclusion} 

We revisit model merging through the lens of MAP inference in the parameter space under a product-of-experts (PoE) energy-based model. We showed that many existing methods can be rediscovered from this new perspective. It enriches the solution space for model merging and at the same time exposes critical limitations of existing work.~In this new view, existing merging methods can be interpreted as imposing a Gaussian structure with light tails on the directional residual between the merged and individual models.~To address this limitation, we introduce Cauchy-based experts that better capture the observed heavy-tailed behavior, resulting in improved performance.~We developed practical algorithms to perform MAP inference on the new PoE design and proved convergence guarantees.

\section*{Acknowledgement}\vspace{-1mm}
This work utilized GPU compute resources at SDSC and ACES through allocation CIS230391 from the Advanced Cyberinfrastructure Coordination Ecosystem:~Services and Support (ACCESS) program~\citep{ACCESS-resource}, which is supported by U.S. National Science Foundation grants $\#$2138259, $\#$2138286, $\#$2138307, $\#$2137603, and $\#$2138296.~Trong Nghia Hoang is supported by the National Science Foundation CAREER Award IIS-2544071.

\bibliography{uai2026-template/uai2026-template/uai2026-template}

\newpage
\appendix
\onecolumn

\section{Proof for Theorem \ref{thm: 1}}
Consider the MAP loss under PoE-EBM \eqref{eq: Cauchy MAP obj}:

\begin{eqnarray}
    \boldsymbol{\zeta}^*
    &=& \arg\min_{\boldsymbol{\zeta}}\ \ \sum_{i=1}^N \log \Big(r_i(\boldsymbol{\zeta})^2+\eta\Big) \ = \ \arg\min_{\boldsymbol{\zeta}}\ \ \sum_{i=1}^N \log \Big([(\boldsymbol{\zeta} - \boldsymbol{\theta}_i)^T\ \boldsymbol{\theta}_i]^2+\eta\Big).
\end{eqnarray}

We write the squared residual term as
\begin{eqnarray}
    r_i(\boldsymbol{\zeta})^2 &=&[(\boldsymbol{\zeta} - \boldsymbol{\theta}_i)^T\ \boldsymbol{\theta}_i]^2 \ = \ (\boldsymbol{\zeta} - \boldsymbol{\theta}_i)^T\ \boldsymbol{\theta}_i\boldsymbol{\theta}_i^\top (\boldsymbol{\zeta} - \boldsymbol{\theta}_i)
\end{eqnarray}
A necessary condition for optimality is that the gradient of the global energy vanishes at $\boldsymbol{\zeta}^*$. This implies the Cauchy score $S^{Cauchy}(\boldsymbol{\zeta}^*)=0$, resulting in the following stationary equation:
\begin{eqnarray}
\sum_{i=1}^N \frac{2r_i(\boldsymbol{\zeta}^*)}{\gamma^2 + r_i(\boldsymbol{\zeta}^*)^2}\boldsymbol{\theta}_i &=& 0.
\end{eqnarray}

\begin{eqnarray}
    S^{Cauchy}(\boldsymbol{\zeta}^*) &=& \sum_{i=1}^N \frac{2\boldsymbol{\theta}_i\boldsymbol{\theta}_i^\top(\boldsymbol{\zeta}^*-\boldsymbol{\theta}_i)}{(\boldsymbol{\zeta}^*-\boldsymbol{\theta}_i)^\top\boldsymbol{\theta}_i\boldsymbol{\theta}_i^\top(\boldsymbol{\zeta}^*-\boldsymbol{\theta}_i)+\gamma} \ = \ 2\sum_{i=1}^N u_i(\boldsymbol{\zeta}^*)\boldsymbol{\theta}_i\boldsymbol{\theta}_i^\top(\boldsymbol{\zeta}^*-\boldsymbol{\theta}_i) \ =\ 0.
\end{eqnarray}
Solving for $\boldsymbol{\zeta}^*$ thus reveals the fixed-point equation as desired,

\begin{eqnarray}
    \boldsymbol{\zeta}^*
    &=&
    \left(
    \sum_{i=1}^N
    u_i(\boldsymbol{\zeta}^*)\,
    \boldsymbol{\theta}_i \boldsymbol{\theta}_i^\top
    \right)^{\dagger}
    \left(
    \sum_{i=1}^N
    u_i(\boldsymbol{\zeta}^*)\,
    \boldsymbol{\theta}_i \boldsymbol{\theta}_i^\top \boldsymbol{\theta}_i
    \right).
\end{eqnarray}

\section{Proof for Theorem \ref{theo: contract}} \label{sec: proof contract}
We prove that the fixed-point mapping
\begin{eqnarray}
    F(\boldsymbol{\zeta}^*) &=& \left(
    \sum_{i=1}^N
    u_i(\boldsymbol{\zeta}^*)\,
    \boldsymbol{\theta}_i \boldsymbol{\theta}_i^\top + \eta \boldsymbol{I}
    \right)^{-1}
    \left(
    \sum_{i=1}^N
    u_i(\boldsymbol{\zeta}^*)\,
    \boldsymbol{\theta}_i \boldsymbol{\theta}_i^\top \boldsymbol{\theta}_i
    \right) \ \ \text{where}\ \ 
\end{eqnarray}
\[
u_i(\boldsymbol{\zeta}^*) \ \ =\ \ \frac{1}{\big[(\boldsymbol{\zeta}^*-\boldsymbol{\theta}_i)^\top \boldsymbol{\theta}_i\big]^2 + \gamma} \ \ \triangleq\ \  \frac{1}{r_i(\boldsymbol{\zeta^\ast})^2 + \gamma},
\qquad \eta>0,\;\gamma>0,
\]
is a contractive map under Assumptions \ref{ass: bounded tv} and \ref{ass: local}, i.e.
\begin{eqnarray} \label{eq: banach cond}
\|F(\boldsymbol{\zeta}^*) - F(\boldsymbol{\nu})\| &\leq& L\|\boldsymbol{\zeta}^*-\boldsymbol{\nu}\| \ \ \text{where} \ \ 0\ \ <\ \ L\ \ <\ \ 1.
\end{eqnarray}
Choosing $\boldsymbol{\nu} = \boldsymbol{\zeta}^{(k)}$ and factoring in that $\boldsymbol{\zeta}^{(k+1)} = F(\boldsymbol{\zeta}^{(k)})$ then leads to the desired result.~For notational ease, we denote
$\boldsymbol{H}(\boldsymbol{\zeta}^*) = \sum_{i=1}^N u_i(\boldsymbol{\zeta}^*) \boldsymbol{\theta}_i\boldsymbol{\theta}_i^\top + \eta \boldsymbol{I}$ and
$\boldsymbol{b}(\boldsymbol{\zeta}^*) = \sum_{i=1}^N u_i(\boldsymbol{\zeta}^*) \boldsymbol{\theta}_i\boldsymbol{\theta}_i^\top \boldsymbol{\theta}_i$. This results in 
\begin{eqnarray} \label{eq: mapping F}
F(\boldsymbol{\zeta}^*) &=& \Big(\boldsymbol{H}(\boldsymbol{\zeta}^*)\Big)^{-1} \boldsymbol{b}(\boldsymbol{\zeta}^*) \ .
\end{eqnarray}

We first need the following lemma for computing the directional gradient of $F$.
\begin{lemma} \label{lemma:gradient F}
    Given the mapping $F$ defined in \eqref{eq: mapping F}. Pick an arbitrary weight $\boldsymbol{\zeta}^*$ and direction $\boldsymbol{h}$. Then the directional gradient of $F$ at $\boldsymbol{\zeta}^*$ in the direction $\boldsymbol{h}$ is given by 
    \begin{eqnarray}
        DF(\boldsymbol{\zeta}^*)_{\boldsymbol{h}} &=& -2 \boldsymbol{H}^{-1} \sum_{i=1}^N \frac{r_i(\boldsymbol{\zeta}^*)}{(r_i(\boldsymbol{\zeta}^*)^2+\gamma)^2}(\boldsymbol{\theta}_i^\top\boldsymbol{h})\Big(\boldsymbol{\theta}_i\boldsymbol{\theta}_i^\top\boldsymbol{\theta}_i - \boldsymbol{\theta}_i\boldsymbol{\theta}_i^\top F(\boldsymbol{\zeta}^*)\Big).
    \end{eqnarray}
\end{lemma}
\begin{proof}
    Applying the definition of directional derivative:
    \begin{eqnarray*}
    DF(\boldsymbol{\zeta}^*)_{\boldsymbol{h}} &=& \frac{d}{dt}F(\boldsymbol{\zeta}^*+t\boldsymbol{h})\Big|_{t=0}.
    \end{eqnarray*}
    Denote $\boldsymbol{H}_t \triangleq \boldsymbol{H}(\boldsymbol{\zeta}^*+t\boldsymbol{h})$ and $\boldsymbol{b}_t \triangleq \boldsymbol{b}(\boldsymbol{\zeta}^*+t\boldsymbol{h})$. Then $F_t \triangleq F(\boldsymbol{\zeta}^*+t\boldsymbol{h}) = \boldsymbol{H}_t^{-1} \boldsymbol{b}_t$. Applying chain rule, we have
    \begin{eqnarray} \label{eq: 30 dF_t}
        \frac{dF_t}{dt} &=& \frac{d}{dt} \boldsymbol{H}^{-1}_t \boldsymbol{b}_t \ + \ \boldsymbol{H}^{-1}_t \frac{d\boldsymbol{b}_t}{dt} \ = \ -\boldsymbol{H}_t^{-1} \frac{d\boldsymbol{H}_t}{dt}\underbrace{\boldsymbol{H}_t^{-1} \boldsymbol{b}_t}_{F_t} + \boldsymbol{H}_t^{-1}\frac{d\boldsymbol{b}_t}{dt} \ ,
    \end{eqnarray}
    where we apply the matrix inverse derivative identity. At $t=0$, \eqref{eq: 30 dF_t} becomes
    \begin{eqnarray} \label{eq: 31 DF}
DF(\boldsymbol{\zeta}^*)_{\boldsymbol{h}} &=& -\boldsymbol{H}(\boldsymbol{\zeta}^*)^{-1} [D\boldsymbol{b}(\boldsymbol{\zeta}^*)_{\boldsymbol{h}} - D\boldsymbol{H}(\boldsymbol{\zeta}^*)_{\boldsymbol{h}}F(\boldsymbol{\zeta}^*)],
    \end{eqnarray}
    where 
    \begin{eqnarray}\label{eq:direct H and b}
          D\boldsymbol{b}(\boldsymbol{\zeta}^*)_{\boldsymbol{h}} &=& \sum_{i=1}^N (\nabla u_i(\boldsymbol{\zeta}^*)^\top \boldsymbol{h}) \boldsymbol{\theta}_i\boldsymbol{\theta}_i^\top \boldsymbol{\theta}_i \ = \ \sum_{i=1}^N\frac{2 r_i(\boldsymbol{\zeta}^*)}{(r_i(\boldsymbol{\zeta}^*)^2+\gamma)^2}(\boldsymbol{\theta}_i^\top\boldsymbol{h}) \boldsymbol{\theta}_i\boldsymbol{\theta}_i^\top\boldsymbol{\theta}_i, \nonumber\\
        D\boldsymbol{H}(\boldsymbol{\zeta}^*)_{\boldsymbol{h}} &=& \sum_{i=1}^N (\nabla u_i(\boldsymbol{\zeta}^*)^\top \boldsymbol{h}) \boldsymbol{\theta}_i\boldsymbol{\theta}_i^\top \ = \ \sum_{i=1}^N\frac{2 r_i(\boldsymbol{\zeta}^*)}{(r_i(\boldsymbol{\zeta}^*)^2+\gamma)^2}(\boldsymbol{\theta}_i^\top\boldsymbol{h}) \boldsymbol{\theta}_i\boldsymbol{\theta}_i^\top.
    \end{eqnarray}
    Substituting \eqref{eq:direct H and b} to \eqref{eq: 31 DF} thus completes the proof.
\end{proof}

It is now sufficient to show that the directional gradient $\|DF(\boldsymbol{\zeta}^*)\|$ of $F$ in an arbitrary direction $\boldsymbol{h}$ in the neighborhood of $\boldsymbol{\zeta}^*$ is less than 1 to show local contraction of $F$~\citep{kreyszig1991introductory}.

First, we compute the gradient of $u_i(\boldsymbol{\zeta}^*)$ as:

\begin{eqnarray} \label{eq: nabla u bound}
    \nabla u_i(\boldsymbol{\zeta}^*) &=&  \frac{2 r_i(\boldsymbol{\zeta}^*)}{(r_i(\boldsymbol{\zeta}^*)^2 + \gamma)^2}\cdot \boldsymbol{\theta}_i
\end{eqnarray}

Next, we define $\boldsymbol{h}=\boldsymbol{\nu} - \boldsymbol{\zeta}^*$ for arbitrary weights $\boldsymbol{\zeta}^*$ and $\boldsymbol{\nu}$ such that $\|\boldsymbol{h}\|=1$ (as per Assumption \ref{ass: local}). We bound the norm of the directional gradients  $DF(\boldsymbol{\zeta}^*)_{\boldsymbol{h}}$ in direction $\mathbf{h}$. Using Lemma \ref{lemma:gradient F}, we can compute $\nabla F(\boldsymbol{\zeta}^*)_{\boldsymbol{h}}$ as
\begin{eqnarray}
DF(\boldsymbol{\zeta}^*)_{\boldsymbol{h}} &=& -2 \boldsymbol{H}^{-1} \sum_{i=1}^N \underbrace{ \frac{r_i(\boldsymbol{\zeta}^*)}{(r_i(\boldsymbol{\zeta}^*)^2+\gamma)^2}\Big(\boldsymbol{\theta}_i\boldsymbol{\theta}_i^\top\boldsymbol{\theta}_i - \boldsymbol{\theta}_i\boldsymbol{\theta}_i^\top F(\boldsymbol{\zeta}^*)\Big)}_{\triangleq\boldsymbol{J}_i}(\boldsymbol{\theta}_i^\top\boldsymbol{h}).
\end{eqnarray}
Therefore, $DF(\boldsymbol{\zeta}^*)_{\boldsymbol{h}}$ is bounded by
\begin{eqnarray}
\|DF(\boldsymbol{\zeta}^*)_{\boldsymbol{h}}\| &=& 2 \Bigg\|\sum_{i=1}^N \boldsymbol{H}^{-1} \boldsymbol{J}_i (\boldsymbol{\theta}_i^\top\boldsymbol{h})\Bigg\| \ \ \leq \ \ 2\sum_{i=1}^N  \| \boldsymbol{H}^{-1} \boldsymbol{J}_i \boldsymbol{\theta}_i^\top\| \ \ \leq \ \ 2M \sum_{i=1}^N \|\boldsymbol{H}^{-1}\boldsymbol{J}_i\| \ ,
\end{eqnarray}
where the first inequality follows from triangle inequality and $\|\boldsymbol{h}\|=1$, the second inequality follows from Assumption \ref{ass: bounded tv}. We empirically show in Figure \ref{fig:Lips_const} that this upper bound is less than 1 in our experiment settings, making $F$ a contracting map around the neighborhood of $\boldsymbol{\zeta}^*$. 

\begin{figure}[t]
    \centering
    \includegraphics[width=0.4\linewidth]{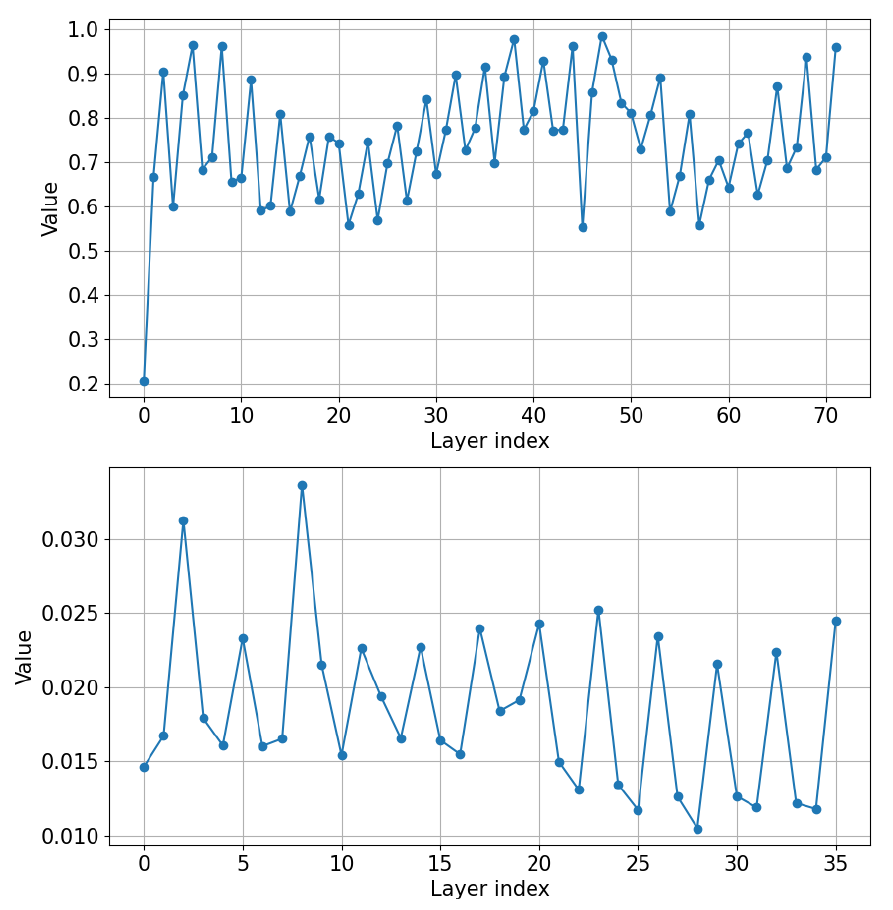}
    \caption{Lipschitz constant values of $F$ across multiple layers of our \texttt{PoE-EBM} when merging 7 fullfinetuned ViT models. Top: ViT-L-14. Bottom: ViT-B-32. The Lipschitz constant is consistently less than 1, indicating that $F$ is a contracting map, ensuring convergence of our algorithm \ref{alg:probabilistic-merging} (see Figure \ref{fig: convergence}).}
    \label{fig:Lips_const}
\end{figure}

\section{Complexity Analysis of Algorithm \ref{alg:probabilistic-merging}.} \label{sec: comp analysis}
Let $\boldsymbol{\theta}_i \in \mathbb{R}^{d\times d}$ be a fine-tuning module. Outer loop over maximum iterations (line 4-12) is executed $T$ times. For each outer iteration, the inner loop (line 5-8) over tasks is executed $N$ times. 

Inner loop complexity:
\begin{enumerate}
    \item Residual computation (line 6): $\mathcal{O}(d^2)$.
    \item Task weight computation (line 7): $\mathcal{O}(1)$.
\end{enumerate}
After the inner loop
\begin{enumerate}
    \item Compute $\boldsymbol{H}$ and $\boldsymbol{b}$ (line 9): $\mathcal{O}(Nd^2)$
    \item Update merged fine-tuning module (line 11): $\mathcal{O}(d^3)$.
\end{enumerate}
Therefore, total complexity is $\mathcal{O}\Bigg(T\Big(N(d^2 + 1) + Nd^2 + d^3\Big)\Bigg) \ \ =\ \  \mathcal{O}(TNd^3)$

\section{Insights into Cauchy score~\ref{eq: score cauchy}} \label{sec: Cauchy insights}
\subsection{Gaussian energy and base score}
 For Gaussian experts \eqref{eq: Gaussian energy} with the precision matrix chosen as $\boldsymbol{\Sigma}_i^{-1} = \frac{1}{\gamma^2} \boldsymbol{\theta}_i\boldsymbol{\theta}_i^\top$, the corresponding energy function, which we refer to as the \textit{base energy}, for the $i$-th task  has the following form:
    \begin{eqnarray}
        E_i^{base}(\boldsymbol{\zeta}) \ \ \triangleq\ \ \frac{1}{2\gamma^2}(\boldsymbol{\zeta}-\boldsymbol{\theta}_i)^\top \boldsymbol{\theta}_i\boldsymbol{\theta}_i^\top(\boldsymbol{\zeta}-\boldsymbol{\theta}_i) 
        &=&  \frac{1}{2\gamma^2}r_i(\boldsymbol{\zeta})^2\ .
    \end{eqnarray}
The global base energy is simply the sum over task energies $E^{base}(\boldsymbol{\zeta}) = \sum_{i=1}^N \frac{1}{2\gamma^2} r_i(\boldsymbol{\zeta})^2$,
which admits the following:
\begin{eqnarray} \label{eq: base score}
    \hspace{-17mm}S^{base}(\boldsymbol{\zeta}) &\triangleq& - \nabla_{\boldsymbol{\zeta}} E^{base} \ =\ - \sum_{i=1}^N \frac{r_i(\boldsymbol{\zeta})}{\gamma^2} \boldsymbol{\theta}_i \ .
\end{eqnarray}

\subsection{Cauchy experts as additive robust guidance}
To understand of the effect of heavy-tail experts, we compare the Cauchy score \eqref{eq: score cauchy} with the base score \eqref{eq: base score}. The Cauchy score \eqref{eq: score cauchy} can be expressed as
\begin{eqnarray*}
    S^{Cauchy} &=& S^{base} \ +\  G(\boldsymbol{\zeta})\ ,
\end{eqnarray*}
where we define the guidance term as:
\begin{eqnarray*}
\hspace{-17mm}    G(\boldsymbol{\zeta}) &\triangleq& \sum_{i=1}^N \Bigg(\frac{r_i(\boldsymbol{\zeta})}{\gamma^2} \ -\ \frac{2r_i(\boldsymbol{\zeta})}{\gamma^2 + r_i(\boldsymbol{\zeta})^2}\Bigg)\boldsymbol{\theta}_i \ .
\end{eqnarray*}
Simplifying the expression yields:
\begin{eqnarray} \label{eq: guidance}
    \hspace{-27mm}G(\boldsymbol{\zeta}) &=& \sum_{i=1}^N \frac{r_i(\boldsymbol{\zeta})\Big(r_i(\boldsymbol{\zeta})^2 - \gamma^2\Big)}{\gamma^2\Big(\gamma^2+r_i(\boldsymbol{\zeta})^2\Big)}\boldsymbol{\theta}_i \ .
\end{eqnarray}
Intuitively, $G(\boldsymbol{\zeta})$ acts as a residual-dependent adjustment that pushes $\boldsymbol{\zeta}$ away from directions with large $|r_i(\boldsymbol{\zeta})|$. In contrast, the base Gaussian score \eqref{eq: base score} amplifies large residuals linearly, pulling the solution toward conflicts regime. The Cauchy
guidance thus provides automatic robustness where it tempers the influence of misaligned tasks while still encouraging alignment
in well-aligned directions.

\section{Limitations}
\texttt{PoE-EBM} currently assumes offline model merging with simultaneous access to all task models and a shared model architecture. Moreover, our formulation focuses on Cauchy experts as the underlying heavy-tailed distribution. Extending the framework to continual merging, alternative heavy-tailed experts, heterogeneous architectures, and multi-modal settings remains potential scopes for future work.

\section{Additional Experiment details} \label{sec: exp details}
\begin{table}[t] 
\centering
\caption{Hyperparameter settings $(\lambda, \gamma, \eta)$ for \texttt{PoE-EBM} when used to merge fine-tuned models in vision and NLP benchmarks.~The numerical algorithm (Section~\ref{sec: merging algo}) used to compute the MAP of \texttt{PoE-EBM} model is configured with convergence tolerance parameter $10^{-5}$ in all scenarios.}
\begin{tabular}{|l l l l l l|}
\hline
\textbf{Domain} & \textbf{Training} & \textbf{Model} & \textbf{Scale $\lambda$} & \boldmath$\gamma$ & $\boldsymbol{\eta}$ \\
\hline
Vision & FFT  & ViT-L/14      & 1.0  & 0.01 & $10^{-3}$ \\
Vision & FFT  & ViT-B/32      & 1.0  & 0.3 & $10^{-3}$ \\
Vision & LoRA & ViT-L/14      & 0.25 & $0.01$ & $10^{-3}$\\
\hline
NLP    & LoRA & Flan-T5-base  & 1.0  & $10^{-3}$ & $10^{-3}$\\
NLP    & LoRA & Flan-T5-large & 1.0  & $10^{-2}$ & $10^{-3}$\\
\hline
\end{tabular} \label{tab: hypers}
\end{table}

\textbf{Tasks.} We conduct our experiments on vision and natural language processing (NLP) tasks. The downstream vision tasks contain the 7-task benchmark: MNIST \citep{lecun1998mnist}, SVHN \citep{netzer2011reading}, Stanford Cars \citep{krause20133d}, DTD \citep{cimpoi2014describing}, GTRSB \citep{stallkamp2011german}, EuroSAT \citep{helber2019eurosat}, Resisc45 \citep{cheng2017remote}. We also include an additional 6 datasets to create a more challenging 13-task benchmark: CIFAR10, CIFAR100 \citep{krizhevsky2009learning}, FashionMNIST \citep{xiao2017fashion}, Flowers102 \citep{nilsback2008automated}, Food \citep{bossard14} and Oxford-IIIT Pet. For NLP task, we use 8 datasets from the GLUE benchmark \citep{wang2018glue}, icluding CoLA, MNLI, MRPC, QNLI, QQP, RTE, SST2 and STSB.

\textbf{Models.} For vision experiments, we leverage pretrained ViT-B/32 and ViT-B/14 models from CLIP \cite{radford2021learning}.
We consider merging models in both fully-finetuned and LoRA fine-tuned ~\citep{hu2022lora} settings. We use the checkpoints provided by~\citet{ilharco2022editing} for fully-finetuned models. The LoRA-finetuned version of these models are provided by~\citet{stoica2024model}. For NLP tasks, we merge LoRA-finetuned Flan-T5-base models and Flan-T5-large models whose checkpoints are provided by~\citet{wei2025modeling}.

\textbf{Metrics.} We report absolute accuracy of the merging methods as well as those of individual fine-tuned models. Following~\citep{ilharco2022editing}, we also report the "normalized accuracy", i.e. the ratio between absolute accuracy of the merged model on task $i$-th and the finetuned model accuracy on the same task. Normalized accuracy shows how close the merged model performs in relative to the finetuned model for each task. Additional experiment details are provided in Appendix~\ref{sec: exp details}.

\textbf{Hyperparameters.} All hyperparameters are selected via extensive grid search on validation performance for each architecture and training regime. We use the same convergence criterion across all experiments and report the best-performing configuration for each setting. Convergence criteria: For all experiments with \texttt{PoE-EBM}, fixed-point iterations are terminated when the norm of the difference between two iterates is smaller or equals to $10^{-5}$. The conditioning parameter $\eta$ is fixed to $10^{-3}$. Other hyperparameter configurations for using \texttt{PoE-EBM} on vision and NLP tasks are reported in Table~\ref{tab: hypers}.

\begin{table*}[t]
\centering
\caption{Average $\ell_2$ norms of full fine-tuning task vectors \(\boldsymbol{\theta}_i\) across seven vision tasks, together with their magnitudes relative to the pretrained model \(\boldsymbol{W}_0\). All values are averaged across layers. The consistently small ratios (<1\%) indicate that task-specific updates remain localized perturbations of the pretrained model, supporting the fine-tuning parameterization adopted throughout this work.
}
\label{tab:l2_ratio}
\begin{tabular}{|l|l| l l l l l l l|}
\hline
Backbone 
& $\ell_2$-norm \& ratio of $\ell_2$-norms
& MNIST 
& SVHN 
& Cars 
& DTD 
& GTSRB 
& EuroSAT 
& RESISC45 \\
\hline

ViT-B-32 & $\|\boldsymbol{\theta_i}\|_2$
& 0.124
 & 0.136 & 0.140 & 0.154 & 0.117 & 0.116 & 0.128 \\
 
& $\|\boldsymbol{\theta}_i\|_2/\|\boldsymbol{W}_0\|_2$ 
& 0.65\% & 0.71\% & 0.73\% & 0.81\% & 0.61\% & 0.61\% & 0.67\% \\

\hline
ViT-L-14 & $\|\boldsymbol{\theta_i}\|_2$
& 0.128 & 0.149 & 0.156 & 0.170 & 0.117 & 0.134 & 0.152 \\

& $\|\boldsymbol{\theta}_i\|_2/\|\boldsymbol{W}_0\|_2$ 
& 0.60\% & 0.70\% & 0.73\% & 0.80\% & 0.55\% & 0.63\% & 0.80\% \\

\hline
\end{tabular}
\end{table*}

\section{Additional results} \label{sec: additional results}
We present additional experimental results and ablation studies in this section. The normalized accuracy for fully fine-tuned ViT-B/32 and ViT-B/14 models is reported in Tables \ref{tab:vitb32_results_normalized} and \ref{tab:vitl14_results_normalized}, respectively. Results for merging LoRA-finetuned and fully fine-tuned ViT-L/14 models on the 7-task vision benchmark are shown in Tables \ref{tab:vitl14lora_results} and \ref{tab:vitl14_results}. We report the merging results on the 13 vision benchmarks in Table \ref{tab: 13 tasks}. Table \ref{tab: Flan-large} summarizes performance on 8 GLUE datasets when merging LoRA-finetuned Flan-T5-large models. Across all settings, \texttt{PoE-EBM} consistently outperforms the compared baselines, including KnOTS \citep{stoica2024model}, which is specifically designed for merging LoRA-finetuned models.~We also provide additional empirical results demonstrating the residuals' heavy-tail behavior when merging Flan-T5 models and ViTs which closely aligns with the Cauchy distribution in Figs.~\ref{fig:vit-heavy} and~\ref{fig:flan-heavy}.

\begin{figure}[t]
    \centering
    \includegraphics[width=\linewidth]{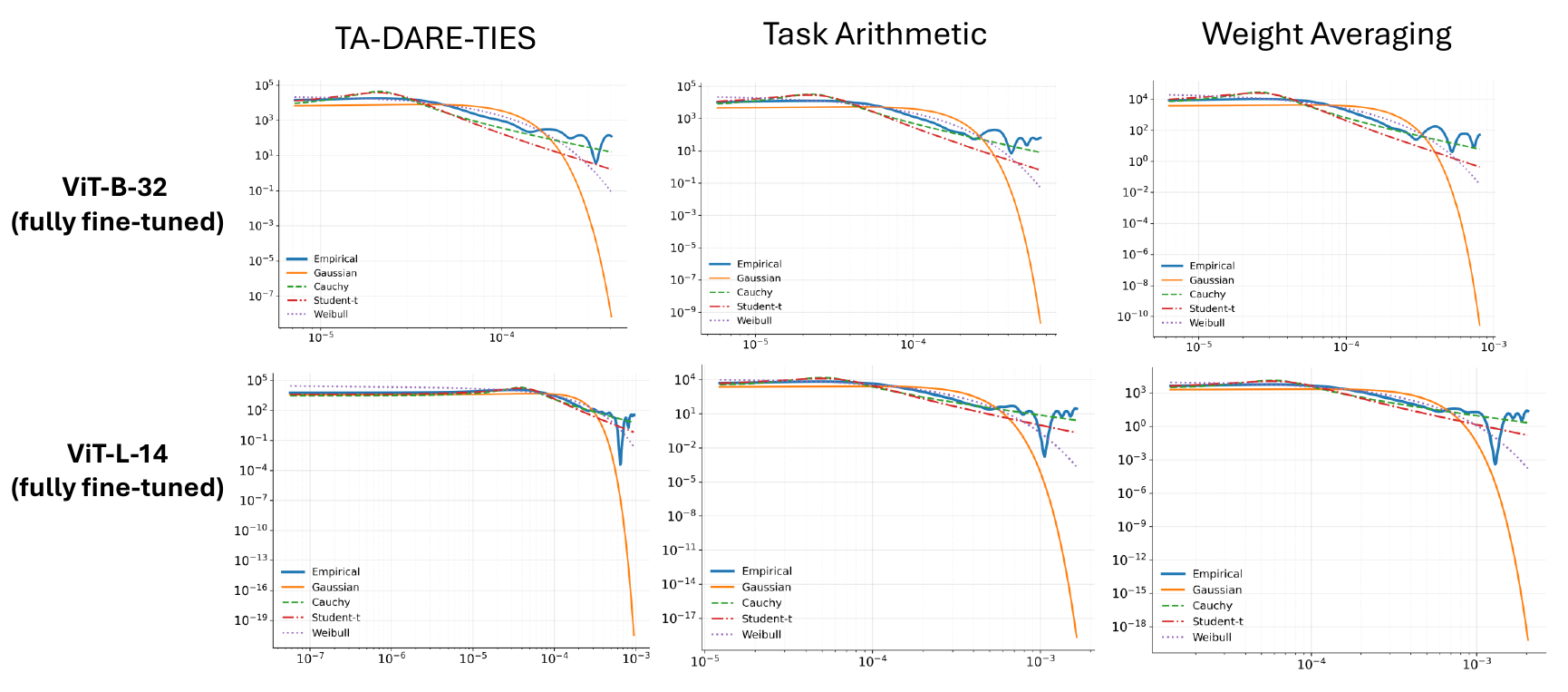}
    \caption{Empirical distributions of directional residuals \(r\) (see Eq.~\ref{eq:dir-res}) produced by different merging methods when merging seven fully fine-tuned ViT models.~Across all methods, the observed tails are substantially heavier than those predicted by light-tailed models and are most accurately captured by a Cauchy distribution among the distributions considered.}
    \label{fig:vit-heavy}
\end{figure}

\begin{figure}[t]
    \centering
    \includegraphics[width=\linewidth]{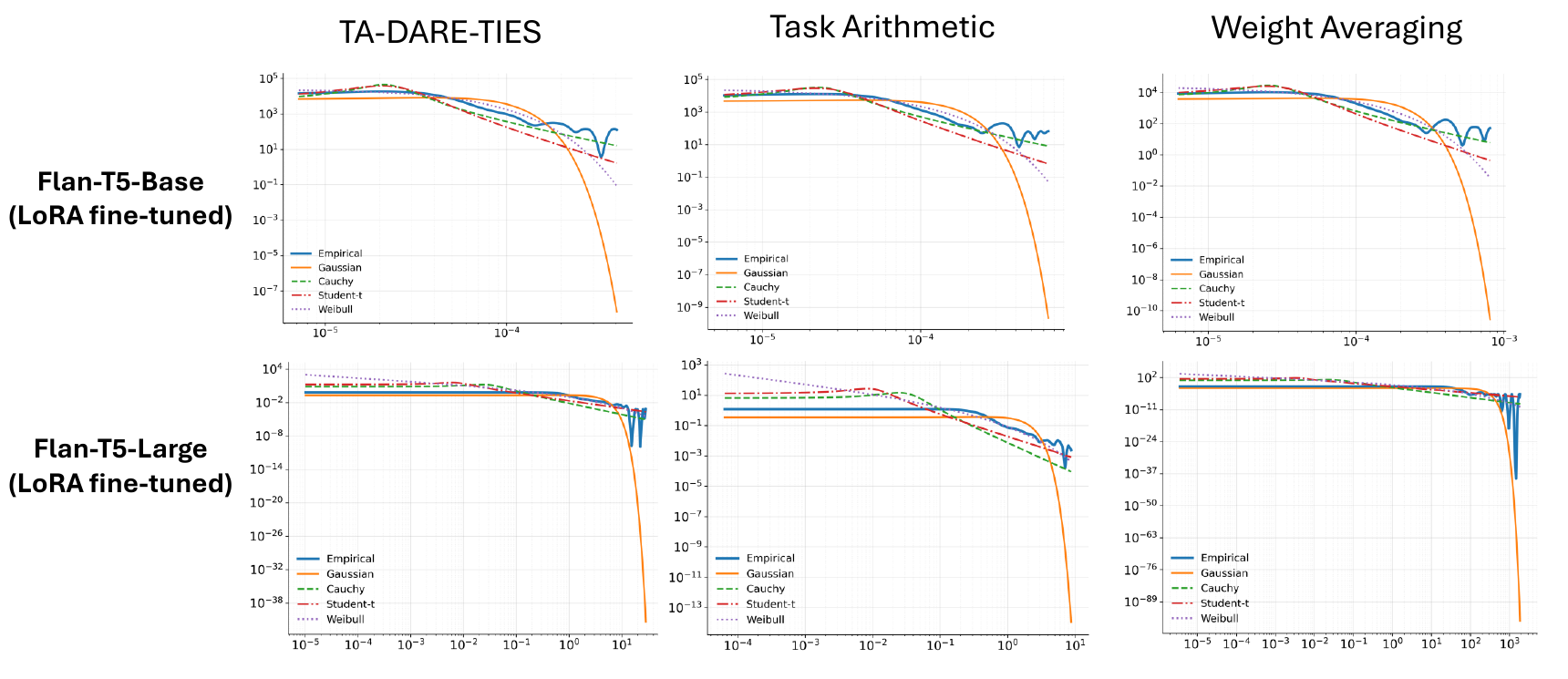}
    \caption{Empirical distributions of directional residuals \(r\) (see Eq.~\ref{eq:dir-res}) produced by different merging methods when merging 8 LoRA fine-tuned Flan-T5 models.}
    \label{fig:flan-heavy}
\end{figure}

\textbf{Ablation on the covariance scale parameter.} We study the effect of the scale parameter $\gamma$ in \eqref{eq:energy cauchy} on the averaged absolute accuracy of our merging algorithm on both ViT-B-32 and ViT-L-14.  We vary the value of $\gamma$, which controls the computation of the task-wise weights $u_i(\cdot)$, over the range $[10^{-3}, 3]$ and plot the corresponding accuracies in Figure \ref{fig: epsilon}. We observe that the performance is relatively stable across a wide range of $\gamma$ values for both ViT models. For ViT-B-32, accuracy peaks at $\gamma=0.3$ and gradually decreases as $\gamma$ increases to $3$. Even at $\gamma=2$, \texttt{PoE-EBM} still outperforms all the baselines compared with accuracy $83.36\%$. For ViT-L-14, the performance remains relatively stable as the accuracy gradually increases with smaller values of $\gamma$. However, at extremely small values (e.g., $\gamma = 10^{-3}$), performance collapses, indicating numerical instability and over-sensitivity in the weight $u_i$'s computation. These results suggest that moderate $\gamma$ values result in stable performance of our merging algorithm, while overly small scales can be detrimental, particularly for larger models.

\begin{figure}[t]
    \centering
    \includegraphics[width=0.7\linewidth]{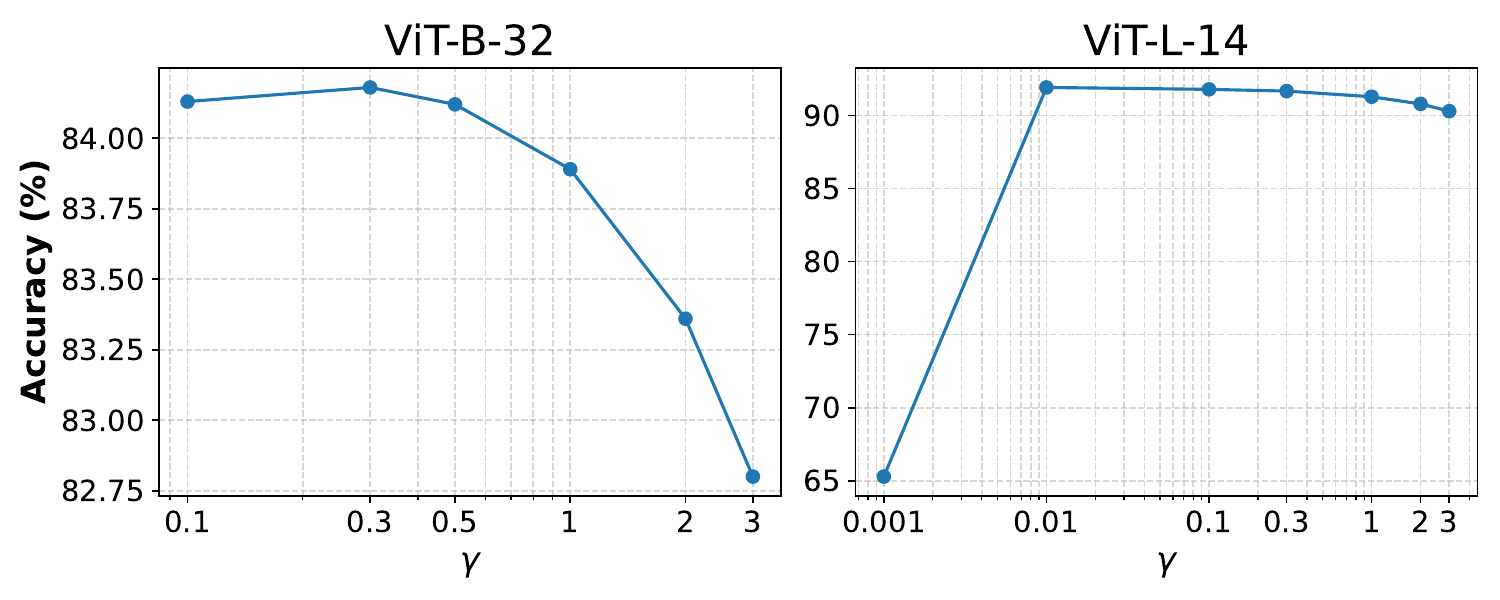}
    \caption{Effect of $\gamma$ values on the merging accuracies of \texttt{PoE-EBM} on the performance accuracy when merging 7 full fine-tuned ViT models.}
    \label{fig: epsilon}
\end{figure}

\begin{table}[t]
\centering
\caption{Reported processing time for our model merging algorithm (see Algorithm~\ref{alg:probabilistic-merging}) across NLP and vision benchmarks.}
\label{tab:merging-time}
\begin{tabular}{|l l l l|}
\hline
\textbf{Domain} & \textbf{Benchmark} & \textbf{Model} & \textbf{Time (s)} \\
\hline
NLP & GLUE (8 tasks) & Flan-T5-base  & 29.26 \\
NLP & GLUE (8 tasks) & Flan-T5-large & 59.19 \\
\hline
Vision & 7 tasks  & ViT-B/32  & 12.10 \\
Vision & 7 tasks  & ViT-L/14  & 42.93 \\
\hline
Vision & 13 tasks & ViT-B/32  & 16.39 \\
Vision & 13 tasks & ViT-L/14  & 50.79 \\
\hline
\end{tabular}
\end{table}

\vspace{-5mm}

\section{Additional Related Works}
To resolve task interference in task arithmetic~\citep{ilharco2022editing}, TIES~\citep{yadav2023ties} improves upon task arithmetic by reducing interference between parameters using their signs and magnitudes before merging.~DARE~\citep{yu2024language} randomly removes fine-tuned weights and rescales the existing ones to create sparse tasks vectors, improving generalization of task arithmetic. Other line of work aims to resolve interference by merging within subspaces.~Alternatively, \citet{ortiz2023task} proposes finetuning models in the tangent space, disentangling finetuned models and thus improving their mergeability. Task Singular Vectors~\citep{gargiulo2025task} combines task vectors using low-rank approximation and reduces interference through means of whitening. KnOTS~\citep{stoica2024model} leverages the Singular Value Decomposition (SVD) of the concatenated task vectors to extract the shared information across all tasks and merge only the task-specific components using task arithmetic. Concrete~\citep{tang2023concrete} uses meta learning to find a common low-dimensional subspace and perform merging with reduced interference. 

More recent works explore model merging through the lens of optimization. DOGE~\citep{wei2025modeling} view model merging as a single constrained optimization problem where the objective is aligning the test performance of the merged model with the task-specific models on their respective tasks. MAP~\citep{li2024map} model the merging problem as a multi-objective optimization problem and aim to identify the Pareto front of the merging coefficients using proxy data.~Nevertheless, these works leverage the geometry of the parameter space and employ non-probabilistic merging scheme and thus do not take into account the uncertainty. On the other hand, our $\texttt{PoE-EBM}$ framework explicitly views the merging problem as probabilistic inference in the parameter space and takes uncertainty into account.

\begin{table*}
\centering
\caption{Multi-task performance comparison when merging ViT-L/14 (full fine-tuned) across 7 vision benchmarks (absolute accuracy).~Task-specific fine-tuning accuracy is colored \textcolor{blue}{blue} to highlight performance upper-bound.}
\label{tab:vitl14_results}
\resizebox{\textwidth}{!}{
\begin{tabular}{|l|l l l l l l l l|}
\hline
& \multicolumn{8}{c|}{\textbf{ViT-L/14}} \\
\cline{2-9}
\textbf{Method} 
& \textbf{MNIST}
& \textbf{SVHN}
& \textbf{Cars}
& \textbf{DTD}
& \textbf{GTSRB}
& \textbf{EuroSAT}
& \textbf{RESISC45}
& \textbf{Average} \\
\hline

\textbf{Finetuning} 
& \textcolor{blue}{99.762}
& \textcolor{blue}{97.881}
& \textcolor{blue}{90.113}
& \textcolor{blue}{97.766}
& \textcolor{blue}{99.129}
& \textcolor{blue}{99.852}
& \textcolor{blue}{96.762}
& \textcolor{blue}{97.320} \\
\hline

Weight averaging
& 79.86
& 59.14
& 75.61
& 54.89
& 57.32
& 54.52
& 68.14
& 64.21 \\

Task Arithmetic (TA)
& 92.90
& 70.65
& 77.80
& 60.27
& 66.38
& 68.93
& 74.33
& 73.04 \\

TA-DARE-TIES
& 98.63
& 88.13
& 80.63
& \underline{72.71}
& 83.79
& 87.78
& 83.51
& 85.03 \\

Fisher Merging
& 83.74
& 62.23
& 77.71
& 57.34
& 58.52
& \textbf{99.04}
& 69.19
& 72.54 \\

DOGE-TA
& 98.88
& \textbf{94.45}
& \textbf{87.85}
& 72.23
& \underline{93.97}
& 96.41
& \textbf{91.94}
& 90.82 \\

Concrete-TA
& \underline{98.99}
& 88.47
& 82.74
& 66.54
& 87.13
& 93.89
& \underline{89.39}
& 86.74 \\
\hline

\texttt{PoE-EBM} (Ours)
& \textbf{99.44}
& \underline{94.41}
& \underline{85.53}
& \textbf{83.56}
& \textbf{94.35}
& \underline{97.00}
& 89.22
& \textbf{91.93} \\

\hline
\end{tabular}
}
\end{table*}

\begin{table*} 
\raggedright
\caption{Multi-task performance comparison when merging Flan-T5-large models (LoRA-fine-tuned) across 8 GLUE benchmarks. Task-specific fine-tuning accuracy is colored \textcolor{blue}{blue} which highlights performance upper bound. Our \texttt{PoE-EBM} achieves the best rank across tasks.}
\label{tab: Flan-large}
\begin{tabular}{|l| l l l l l l l l l l|}
\hline
\textbf{Method}
& \textbf{CoLA}
& \textbf{MNLI}
& \textbf{MRPC}
& \textbf{QNLI}
& \textbf{QQP}
& \textbf{RTE}
& \textbf{SST-2}
& \textbf{STS-B}
& \textbf{Average}
& \textbf{Rank} \\

\hline

\textbf{Finetuning}
& \textcolor{blue}{80.20}
& \textcolor{blue}{88.51}
& \textcolor{blue}{89.23}
& \textcolor{blue}{94.40}
& \textcolor{blue}{87.18}
& \textcolor{blue}{91.74}
& \textcolor{blue}{95.19}
& \textcolor{blue}{90.91}
& \textcolor{blue}{89.67}
& \textcolor{blue}{NA} \\
\hline

Weight Averaging
& 74.59
& 84.28
& 84.07
& 92.79
& \textbf{86.28}
& 87.36
& 94.84
& \underline{87.98}
& 86.52
& 4.00 \\

Task Arithmetic
& 76.89
& 85.44
& 85.29
& 93.92
& 85.84
& \underline{88.09}
& \underline{95.18}
& 87.75
& 87.30
& 3.38 \\

TIES-Merging
& 75.55
& 84.69
& 84.31
& \textbf{93.94}
& \underline{86.18}
& \textbf{88.45}
& 95.07
& 87.82
& 86.93 
& 3.13 \\

Concrete-TA
& 76.89
& 86.16
& \textbf{88.54}
& \underline{93.92}
& 85.84
& \underline{88.09}
& \underline{95.18}
& 87.91
& 87.44
& \underline{2.50} \\

DOGE-TA
& \textbf{78.12}
& \underline{88.08}
& 86.52
& 93.80
& 85.82
& 86.72
& 95.00
& 87.71
& \underline{87.72}
& 3.63 \\

\hline

\texttt{PoE-EBM} (Ours)
& \underline{77.28}
& \textbf{88.33}
& \underline{86.76}
& 93.06
& 85.98
& 87.36
& \textbf{95.30}
& \textbf{88.98}
& \textbf{87.88}
& \textbf{2.25} \\

\hline
\end{tabular}
\end{table*}

\begin{table*}
\centering
\caption{Multi-task performance comparison when merging ViT-B/32 (full fine-tuned) across 7 vision benchmarks. The performance is reported in terms of normalized accuracy.}
\label{tab:vitb32_results_normalized}
\resizebox{\textwidth}{!}{
\begin{tabular}{|l| l l l l l l l l|}
\hline
& \multicolumn{8}{c|}{\textbf{ViT-B/32}} \\
\cline{2-9}
\textbf{Method} 
& \textbf{MNIST}
& \textbf{SVHN}
& \textbf{Cars}
& \textbf{DTD}
& \textbf{GTSRB}
& \textbf{EuroSAT}
& \textbf{RESISC45}
& \textbf{Average} \\
\hline

\textbf{Finetuning} 
& \textcolor{blue}{100.00}
& \textcolor{blue}{100.00}
& \textcolor{blue}{100.00}
& \textcolor{blue}{100.00}
& \textcolor{blue}{100.00}
& \textcolor{blue}{100.00}
& \textcolor{blue}{100.00}
& \textcolor{blue}{100.00} \\
\hline

Weight averaging
& 85.28
& 65.71
& 78.71
& 53.20
& 55.89
& 63.59
& 71.56
& 67.70 \\

Task Arithmetic (TA)
& 90.95
& 73.52
& 78.81
& 56.53
& 63.25
& 66.30
& 72.53
& 71.70 \\

TA-DARE-TIES
& 96.29
& 84.16
& 77.81
& 62.60
& 71.48
& 70.64
& 70.75
& 76.25 \\

Fisher Merging
& 90.06
& 84.49
& \underline{90.56}
& 56.04
& 58.50
& 77.28
& 77.57
& 75.91 \\

DOGE-TA
& \underline{98.73}
& \underline{89.80}
& \textbf{92.37}
& \underline{66.10}
& \underline{88.55}
& \underline{90.13}
& \textbf{86.30}
& 87.42 \\

Concrete-TA
& 97.31
& 82.09
& 75.95
& 60.20
& 71.92
& 76.39
& 75.01
& 77.08 \\
\hline

\texttt{PoE-EBM} (Ours)
& \textbf{99.20}
& \textbf{94.03}
& 86.34
& \textbf{78.19}
& \textbf{89.60}
& \textbf{90.24}
& \underline{81.64}
& \textbf{88.46} \\

\hline
\end{tabular}
}
\end{table*}

\begin{table*}[t]
\centering
\caption{Multi-task performance comparison when merging ViT-L/14 (full fine-tuned) across 7 vision benchmarks.~The performance is reported in terms of normalized accuracy.}
\label{tab:vitl14_results_normalized}
\resizebox{\textwidth}{!}{
\begin{tabular}{|l| l l l l l l l l|}
\hline
& \multicolumn{8}{c|}{\textbf{ViT-B/32}} \\
\cline{2-9}
\textbf{Method} 
& \textbf{MNIST}
& \textbf{SVHN}
& \textbf{Cars}
& \textbf{DTD}
& \textbf{GTSRB}
& \textbf{EuroSAT}
& \textbf{RESISC45}
& \textbf{Average} \\
\hline

\textbf{Finetuning} 
& \textcolor{blue}{100.00}
& \textcolor{blue}{100.00}
& \textcolor{blue}{100.00}
& \textcolor{blue}{100.00}
& \textcolor{blue}{100.00}
& \textcolor{blue}{100.00}
& \textcolor{blue}{100.00}
& \textcolor{blue}{100.00} \\
\hline

Weight averaging
& 80.05
& 60.42
& 83.91
& 56.15
& 57.83
& 54.60
& 70.42
& 66.20 \\

Task Arithmetic (TA)
& 93.12
& 72.18
& 86.34
& 61.64
& 66.96
& 69.03
& 76.82
& 75.16 \\

TA-DARE-TIES
& 98.86
& 90.04
& 89.48
& \underline{74.37}
& 84.43
& 87.91
& 86.30
& 87.36 \\

Fisher Merging
& 83.94
& 63.58
& 86.24
& 58.65
& 59.03
& \textbf{99.18}
& 71.51
& 74.89 \\

DOGE-TA
& 99.12
& \textbf{96.49}
& \textbf{97.49}
& 73.88
& \underline{94.80}
& 96.55
& \textbf{95.01}
& 93.34 \\

Concrete-TA
& \underline{99.23}
& 90.38
& 91.82
& 68.06
& 87.90
& 94.03
& \underline{92.38}
& 89.12 \\
\hline

\texttt{PoE-EBM} (Ours)
& \textbf{99.68}
& \underline{96.45}
& \underline{94.91}
& \textbf{85.47}
& \textbf{95.18}
& \underline{97.14}
& 92.21
& \textbf{94.43} \\

\hline
\end{tabular}
}
\end{table*}

\begin{table*}[t]
\centering
\caption{Multi-task performance comparison when merging ViT-L/14 (LoRA-fine-tuned) across 7 vision benchmarks.}
\label{tab:vitl14lora_results}
\resizebox{\textwidth}{!}{
\begin{tabular}{|l| l| l l l l l l l |l|}
\hline
& & \multicolumn{8}{c|}{\textbf{ViT-L/14}} \\
\cline{3-10}
\textbf{Method} & \textbf{Accuracy Type}
& \textbf{MNIST}
& \textbf{SVHN}
& \textbf{Cars}
& \textbf{DTD}
& \textbf{GTSRB}
& \textbf{EuroSAT}
& \textbf{RESISC45}
& \textbf{Average} \\
\hline

\multicolumn{2}{|l|}{\textbf{Finetuning}} 
& \textcolor{blue}{99.53}
& \textcolor{blue}{97.72}
& \textcolor{blue}{99.77}
& \textcolor{blue}{70.05}
& \textcolor{blue}{97.20}
& \textcolor{blue}{98.59}
& \textcolor{blue}{95.70}
& \textcolor{blue}{94.08} \\
\hline

\multirow{2}{*}{Weight averaging}
& Absolute
& 78.55
& 60.19
& 78.22
& 55.80
& 52.19
& \textbf{62.74}
& 72.40
& 65.73 \\
& Normalized
& 78.93
& 61.59
& 78.40
& 79.65
& 53.69
& \textbf{63.64}
& 75.65
& 70.22 \\
\hline

\multirow{2}{*}{Task Arithmetic (TA)}
& Absolute
& 77.93
& 59.70
& 78.14
& 55.69
& 51.85
& \underline{62.70}
& 72.06
& 65.44 \\
& Normalized
& 78.30
& 61.09
& 78.33
& 79.50
& 53.35
& \underline{63.60}
& 75.30
& 69.62 \\
\hline

\multirow{2}{*}{TA-DARE-TIES}
& Absolute
& 70.86
& 71.48
& 81.86
& 57.45
& 60.38
& 57.22
& 79.13
& 68.30 \\
& Normalized
& 71.20
& 73.14
& 81.75
& 82.01
& 62.12
& 58.04
& 82.68
& 72.99 \\
\hline

\multirow{2}{*}{KnOTS-TIES}
& Absolute
& 81.71
& \underline{75.38}
& \textbf{83.83}
& \textbf{58.30}
& \textbf{68.66}
& 61.74
& 79.79
& 72.77 \\
& Normalized
& 82.10
& \underline{77.14}
& \textbf{84.03}
& \textbf{83.22}
& \textbf{70.63}
& 62.62
& 83.38
& 77.59 \\
\hline

\multirow{2}{*}{KnOTS-DARE-TIES}
& Absolute
& 67.43
& 67.06
& 82.68
& 58.19
& 64.13
& 59.67
& \underline{80.00}
& 68.45 \\
& Normalized
& 67.75
& 68.62
& 82.88
& 83.07
& 65.98
& 60.52
& \underline{83.60}
& 73.20 \\
\hline
\hline

\multirow{2}{*}{\texttt{PoE-EBM} (Ours)}
& Absolute
& \textbf{89.56}
& \textbf{75.56}
& \underline{83.37}
& 56.97
& \underline{66.40}
& 60.15
& \textbf{80.37}
& \textbf{73.20} \\
& Normalized
& \textbf{89.99}
& \textbf{77.32}
& \underline{83.56}
& 81.32
& \underline{68.31}
& 61.00
& \textbf{83.98}
& \textbf{77.93} \\

\hline
\end{tabular}
}
\end{table*}

\end{document}